\documentclass[a4paper,11pt]{article}
\pdfoutput=1
\usepackage{amsmath}
\usepackage{amssymb}
\usepackage{natbib}
\usepackage{graphicx}
\usepackage{algorithm}
\usepackage{algorithmic}
\usepackage{natbib}
\usepackage{url}
\usepackage{tikz}
\usepackage{wrapfig}
\usepackage{subfigure}
\usepackage{comment}
\usepackage{sidecap}
\usepackage{geometry}
\newcommand{\prob}{{\mathbf P}}
\newcommand{\supp}{{\mathrm{supp}\,}}
\newcommand{\argmax}{{\mathrm{argmax}\,}}
\newcommand{\argmin}{{\mathrm{argmin}\,}}
\newcommand{\Z}{{\mathbb Z}}
\newcommand{\N}{{\mathbb N}}
  
\newcommand{\im}{\mathrm{Im}}
\newcommand{\DM}{\mathrm{DM}}

\newenvironment{proof}{{\bf Proof.}}{\hfill$\square$\vskip\baselineskip}
\newtheorem{example}{Example}

\newtheorem{theorem}{Theorem}
\newtheorem{cor}[theorem]{Corollary}
\newtheorem{prop}[theorem]{Proposition}
\newtheorem{definition}[theorem]{Definition}

\newcommand\independent{\protect\mathpalette{\protect\independenT}{\perp}}
\def\independenT#1#2{\mathrel{\rlap{$#1#2$}\mkern2mu{#1#2}}}
\newcommand{\notindependent}{\independent \hspace{-4.1mm} \diagup \hspace{1.4mm}}

\begin{document}

\title{Causal Inference on Discrete Data using\\ Additive Noise Models}

\author{Jonas Peters \thanks{jonas.peters@tuebingen.mpg.de}\hspace{0.5cm}
 Dominik Janzing \thanks{dominik.janzing@tuebingen.mpg.de} \hspace{0.5cm}
 Bernhard Sch\"olkopf \thanks{bernhard.schoelkopf@tuebingen.mpg.de} \vspace{0.5cm}\\
 MPI for Biological Cybernetics\\
 Spemannstr. 38\\
 72076 T\"ubingen, Germany}
\date{}

\maketitle

\begin{abstract}
Inferring the causal structure of a set of random variables from a finite sample of the joint distribution is an important problem in science. Recently, methods using additive noise models have been suggested to approach the case of continuous variables.
In many situations, however, the variables of interest are discrete or even have only finitely many states.
In this work we extend the notion of additive noise models to these cases.
We prove that whenever the joint distribution $\prob^{(X,Y)}$ admits such a model in one direction, e.g. $Y=f(X)+N, \; N \independent X$, it does not admit the reversed model $X=g(Y)+\tilde N, \; \tilde N \independent Y$ as long as the model is chosen in a generic way.
Based on these deliberations we propose an efficient new algorithm that is able to distinguish between cause and effect for a finite sample of discrete variables. In an extensive experimental study we show that this algorithm works both on synthetic and real data sets.
\end{abstract}

\section{Introduction}
Inferring causal relations 
by analyzing statistical dependences among observed random variables
 is a challenging task if no controlled randomized experiments are available.
So-called constraint-based approaches to causal discovery \citep{Pearl00,Spirtes} 
select among all directed acyclic graphs (DAGs) those that satisfy the Markov condition and the faithfulness assumption, i.e., those for which the observed independences are imposed by the structure rather than being a result of specific choices of parameters of the Bayesian network.
These approaches are unable to distinguish among causal DAGs that impose the same independences. In particular,
it is impossible to distinguish between $X\rightarrow Y$ and $Y\rightarrow X$.
More recently, several methods have been suggested that use not only conditional independences, but also
more sophisticated properties of the joint distribution. For simplicity, we explain the ideas for the two variable setting since this case is particularly challenging.
 \cite{Kano2003} use models
\begin{equation}\label{anoise}
Y=f(X)+ N
\end{equation}
where $f$ is a linear function and $N$ is additive noise that is independent of the hypothetical cause $X$. This is an example for an additive noise model from $X$ to $Y$. 
Apart from trivial cases, $P(X,Y)$ can only admit such a model from $X$ to $Y$ and from $Y$ to $X$ in the bivariate Gaussian case.
\citet{Hoyer} generalize the method to non-linear functions $f$ and showed that generic models of this form generate joint distributions that do not admit such an additive noise model from $Y$ to $X$.
\citet{Zhang} augment the model by applying a non-linear function $g$ to the rhs of eq.~(\ref{anoise})
and still obtain identifiability for generic cases. 
\citet{peters09} use independent linear additive noise models in order to detect whether a sample of a time series has been reversed. Their positive results further support this way of causal reasoning.  
All these proposals, however, were only designed for real-valued variables $X$ and $Y$.

For discrete variables, \cite{Neurocomputing} propose a method to measure the complexity of causal models
via a Hilbert space norm of the logarithm of conditional densities and prefer models that induce smaller norms.
\cite{SunLauderdale} fit joint distributions of cause and effect with conditional densities whose logarithm is
a second order polynomial (up to the log-partition function) and show that this often makes causal directions
identifiable when some or all variables are discrete.
For discrete variables, several Bayesian approaches \citep{Heckerman} are also applicable, but the construction of good priors are challenging and often the latter are designed such that Markov equivalent DAGs still remain indistinguishable.

Here, we extend the model in eq.~(\ref{anoise}) to the discrete case in two different ways: (I) If both $X$ and $Y$ take values in $\Z$ (the support may be finite, though) additive noise models can be defined analogously to the continuous case. (II) If both $X$ and $Y$ take only finitely many values we can still define additive noise models by interpreting the $+$ sign as an addition in the finite ring $\Z/m\Z$. We propose to apply this method to variables where the cyclic structure is appropriate
(e.g., the direction of the wind after discretization, day of the year, season).
However, the applicability of this second model class is not restricted to random variables that take integers as values: 
Assume that $X$ and $Y$ take values in $\mathcal A:=\{a_1, \ldots, a_m\}$ and $\mathcal B:=\{b_1, \ldots, b_{\tilde m}\}$, which are structureless sets. Considering functions $f:\mathcal A\rightarrow \mathcal B$ and models with
$\prob(Y=b_j \,|\,X=a_i)= p$ if $b_j=f(a_i)$ and $(1-p)/(m-1)$ otherwise, is
a special case of an additive noise model: Impose any cyclic structure on the data and use the additive noise $\prob(N=0)=p, \prob(N=l)=(1-p)/(m-1)$ for $l\neq 0$. 
This may be helpful whenever the random variables are categorical and when these categories do not inherit any kind of ordering (e.g. different treatments of organisms or phenotypes). 
\\
In the following article we refer to (I) by saying {\it integer constraint}, whereas model (II) satisfies the {\it cyclic constraint}.


The main idea of the causal inference method we propose goes as follows: If such an additive noise model exists in one direction but not in the other, we prefer the
former based on Occam's Razor and infer it to be the causal direction. 

Such a procedure is only sensible if there are only few instances, in which there is an additive noise models in both directions. If, for example, all additive noise models from $X$ to $Y$ also allow an additive noise model from $Y$ to $X$, we could not draw any causal conclusions at all. We will show that {\it reversible} cases are very rare and thereby answer this theoretical question.

For a practical causal inference method we have to test whether the data admits an additive noise model and thus have to perform a discrete regression. But since in the discrete case regularization of the regression function is not necessary, in principle we would have to check all possible functions and test whether they result in independent residuals. This is highly intractable, of course, and we therefore propose an efficient heuristic procedure that proved to work very well in practice.

In section 2 we extend the concept of additive noise models to discrete random variables and show the corresponding identifiability results for generic cases in section 3. In section 4 we introduce an efficient algorithm for causal inference on finite data, for which we show experimental results in section 5. We conclude in section 6. 

\section{Additive Noise Models for Discrete Variables}
As it has been proposed for the continuous case by \citet{Shimizu2006, Hoyer, Zhang} we assume the following causal principle to hold throughout the remainder of this article:

{\bf Causal Inference Principle (for discrete random variables)} \hspace{0.1cm} {\it Whenever $Y$ satisfies an additive noise model with respect to $X$ and not vice versa then $X$ is a cause for $Y$, and we write $X \rightarrow Y$.}  

\citet{bastian} give further theoretical support for this principle using the concept of Kolmogorov complexity 
and \citet{peters09} use this way of reasoning for detecting the arrow of time. 

Note that whenever there is no additive noise model in any direction (which may well happen) we do not draw any causal conclusions and other causal inference methods should be tried. 

We now precisely explain what we mean by an additive noise model in the case of discrete random variables. For simplicity we denote $p_X(x)=\prob(X=x)$, $p_Y(y)=\prob(Y=y)$,  $n(l)=\prob(N=l)$ and $\tilde n(k)=\prob(\tilde N=k)$ and $\supp X$ is defined as the set of all values that $X$ takes with probability larger than $0$: $\supp X:=\{k\,|\,p_X(k)>0\}$.

\subsection{Integer Constraint}
Assume that $X$ and $Y$ take values in $\Z$ (their distributions may have finite support). We say that there is an additive noise model (ANM) from $X$ to $Y$ if
$$
Y=f(X)+N\,, \quad N \independent X
$$
where $f: \Z \rightarrow \Z$ is an arbitrary function and $N$ a noise variable that takes integers as values, too.

Furthermore we require $n(0)\geq n(j)$ for all $j\neq 0$.
This does not restrict the model class, but is due to a freedom we have in choosing $f$ and $N$: If $Y=f(X)+N, \, N \independent X$, then we can always construct a new function $f_j$, such that $Y=f_j(X)+N_j, \, N_j \independent X$ by choosing $f_j(i)=f(i)+j$ and $n_j(i)=n(i+j)$.

Such an additive noise model is called {\it reversible} if there is also an additive noise model from $Y$ to $X$, i.e. if it satisfies an additive noise model in both directions.

\subsection{Cyclic Constraint}
We can extend
additive noise models to random variables which inherit a cyclic structure and therefore take values in a periodic domain. 
Random variables are usually defined as measurable maps from a probability space into the real numbers. We thus make the following definition
\begin{definition}
Let $(\Omega, \mathcal F, \prob)$ be a probability space. A function $X: \Omega \rightarrow \mathbf Z/m\mathbf Z$ is called an {\rm $m$-cyclic random variable} if $X^{-1}(k) \in \mathcal F \; \forall k \in \mathbf Z/m\mathbf Z$. 
All other concepts of probability theory (like distributions and expectations) can be constructed analogous to the well-known case, in which $X$ takes values $\{0, \ldots, m-1\}$.
\end{definition}

Let $X$ and $Y$ be $m$- and $\tilde m$-cyclic random variables. We say that $Y$ satisfies an additive noise model from $X$ to $Y$ if there is a function $f:\mathbf Z/m\mathbf Z \rightarrow \mathbf Z/{\tilde m}\mathbf Z$ and an $\tilde m$-cyclic noise $N$ such that
$$
Y=f(X)+N\; \mbox{ and }N \independent X.
$$
Again we require $n(0)\geq n(j)$ for all $j\neq 0$ and call this model {\it reversible} if there is a function $g:\mathbf Z/{\tilde m}\mathbf Z \rightarrow \mathbf Z/m\mathbf Z$ and an $m$-cyclic noise $\tilde N$ such that
$$
X=g(Y)+\tilde N\; \mbox{ and }\tilde N \independent Y.
$$

\subsection{Relations}
The following two remarks are essential in order to understand the relationship between integer and cyclic constraints:

(1) The difference between these two models manifests in the target domain. If we consider an ANM from $X$ to $Y$ it is important whether we put integer or cyclic constraints on $Y$ (and thus on $N$). It does not make a difference, however, whether we consider the regressor $X$ to be cyclic (with a cycle larger than the support of $X$) or not. The independence constraint remains the same.

(2) In the finite case additive noise models with cyclic constraints are more general than the ones with integer constraints: Assume there is an ANM $Y=f(X)+N$, where all variables are taken to be non-cyclic and $Y$ takes values between $k$ and $l$, say. 
Then we still have an ANM $Y=f(X)+N$ if we regard $Y$ to be $l-k+1$-cyclic because $N \mod (l-k+1)$ remains independent of $X$. It is possible, however, that $N \notindependent X$, but $N \mod (l-k+1) \independent X$
(see Example 2 in Section 3.2.1).\\
\begin{wrapfigure}[10]{r} {6.5cm}
\begin{center}
\begin{tikzpicture}[scale=1.1]
   \draw (-2.5,1.5) rectangle (2.5,-1.5);
   \begin{scope}
      \clip (-0.8,-0.4) ellipse (1.5cm and 0.7cm);
      \clip (0.8,0.4) circle (1.5cm and 0.7cm);
      \fill[color=gray] (-2,1.5) rectangle (2,-1.5);
   \end{scope}
   \draw (-0.8,-0.4) ellipse (1.5cm and 0.7cm);
   \draw (0.8,0.4) circle (1.5cm and 0.7cm);
   \draw (-1,-0.5) node {$F$};
   \draw (1,0.5) node {$B$};
   \draw (1.5,-1) node {$A$};
\end{tikzpicture}
\end{center}
\caption{How large is $F \cap B$?} 
\label{intersection}
\end{wrapfigure}
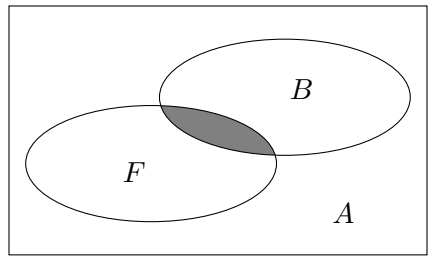

\section{Identifiability} \label{sec_theorie}
Whether an additive noise model is allowed depends on the form of the joint distribution $\prob^{(X,Y)}$. 
Let $F$ be the subset of the set $A$ of all possible joint distributions that allow an additive noise
model from $X$ to $Y$ in the ``forward direction'', whereas $B$ allows an additive noise model in the backward direction from $Y$ to $X$. Some trivial examples like $p_X(0)=1, n(0)=1$ and $f(0)=2$ immediately show that there are joint distributions allowing additive noise models in both directions, meaning $F \cap B \neq \emptyset$ (see Figure \ref{intersection}). 
But how large is this intersection? Our method would not be useful if we find out that $F$ and $B$ are almost the same sets. Then most additive noise models can be fit in either both directions or in none. For additive noise models with 
integer constraints and with cyclic constraints we identify the intersection $B \cap F$ and show that it is indeed a very small set. If we are unlucky and the data generating process we consider happens to be in $B \cap F$, our method does not give wrong results, but answers ``I do not know the answer''. In all other situations the method identifies the correct direction given that we observe enough data. The proofs are provided in the appendix.\\

\subsection{Integer Constraint}  \label{sec_non_cyclic_id}
\subsubsection{$Y$ or $X$ has finite support}
First we assume that either the support of $X$ or the support of $Y$ is finite. This already covers the situation in most applications.
Figure \ref{ex1} (the dots indicate a probability greater than 0) shows an example of a joint distribution that allows an ANM from $X$ to $Y$, but not from $Y$ to $X$. This can be seen easily at the ``corners'' $X=1$ and $X=7$: Whatever we choose for $g(0)$ and $g(4)$, the distribution of $\tilde N\,|\,Y=0$ is supported only by one point, whereas $\tilde N\,|\,Y=4$ is supported by 3 points. Thus $\tilde N$ cannot be independent of $Y$.

Figure \ref{counter_finite} shows a (rather non-generic) example that allows an ANM in both directions if we choose $p_X(a_i)=\frac{1}{36}, p_X(b_i)=\frac{2}{36}$ for $i=1, \ldots, 4$ and $p_X(a_i)=\frac{2}{36}, p_X(b_i)=\frac{4}{36}$ for $i=5, \ldots, 8$. 

We prove the following
\begin{theorem} \label{yfinite}
An additive noise model $X \rightarrow Y$ is reversible $\Longleftrightarrow$ there exists a disjoint decomposition $\bigcup_{i=1}^l C_i=\supp X$, such that 
\begin{itemize}
\item The $C_i$s are shifted versions of each other
$$\forall i \,\exists d_i\geq0\,:\, C_i = C_0 + d_i$$
and $f$ is piecewise constant:
$$f\mid_{C_{i}} \equiv c_i \; \forall i.$$
\item The probability distributions on the $C_i$s are shifted and scaled versions of each other with the same shift constant as above: For $x \in C_i$
$$\prob(X=x) = \prob(X=x-d_i) \cdot \frac{\prob(X \in C_i)}{\prob(X \in C_0)}$$
holds. 
\item The sets $c_i + \supp N:=\{c_i+h\,:\,n(h)>0\}$ are disjoint. 
\end{itemize}
\end{theorem}
Obviously, such a decomposition for $\supp Y$ that satisfies the same criteria must exist, too (symmetry argument). We are now given a full characterization of all cases that allow an ANM in both directions. Even each of the conditions by itself is very restrictive, so that all conditions together describe a very small class of models: in almost all cases the direction of the model is identifiable. In Figure \ref{counter_finite} all $a_i$ belong to $C_0$, all $b_j$ to $C_1$ and $d_1=1$.

As for the other theorems of this section the proof is provided in the appendix. Its main point is based on the asymmetric effects of the ``corners'' of the joint distribution. In order to allow for an infinite support of $X$ (or $Y$) the proof generalizes the concept of ``corners''.

\begin{figure}[h]
\begin{center}
\begin{minipage}[t]{0.36\textwidth}
\centering
\begin{tikzpicture}[scale=0.37,inner sep=1.55mm]
  \draw[gray,very thin] (-0.3,-0.3) grid (8.3,8.3);
  \draw[->] (-1,0) -- (8.5,0) node[right] {$X$};
  \draw[->] (0,-1) -- (0,8.5) node[above] {$Y$};
  \foreach \i in {0,1,2,3,4,5,6}
  \foreach \j in {0,1,2}
    {\fill (\i+1,\i+\j) circle (0.2cm);}
  \draw[shift={(2,0)}] (0pt,2pt) -- (0pt,-2pt) node[below] {$2$};
  \draw[shift={(4,0)}] (0pt,2pt) -- (0pt,-2pt) node[below] {$4$};
  \draw[shift={(6,0)}] (0pt,2pt) -- (0pt,-2pt) node[below] {$6$};
  \draw[shift={(8,0)}] (0pt,2pt) -- (0pt,-2pt) node[below] {$8$};
  \draw[shift={(0,2)}] (-3pt,0pt) -- (2pt,0pt) node[left] {$2$};
  \draw[shift={(0,4)}] (-3pt,0pt) -- (2pt,0pt) node[left] {$4$};
  \draw[shift={(0,6)}] (-3pt,0pt) -- (2pt,0pt) node[left] {$6$};
  \draw[shift={(0,8)}] (-3pt,0pt) -- (2pt,0pt) node[left] {$8$};
\end{tikzpicture}
\caption{This joint distribution satisfies an additive noise model only from $X$ to $Y$.
} 
\label{ex1}
\end{minipage}
\hspace{0.03\textwidth}
\begin{minipage}[t]{0.57\textwidth}
\centering
\begin{tikzpicture}[scale=0.37,inner sep=1.55mm]
  \draw[gray,very thin] (-0.3,-0.3) grid (20.3,8.3);
  \draw[->] (-1,0) -- (20.5,0) node[right] {$X$};
  \draw[->] (0,-1) -- (0,8.5) node[above] {$Y$};
  \foreach \i in {0,1}
  \foreach \j in {0,1,2}
    {\fill (5*\i+1,1+2*\j) circle (0.13cm);
    \fill (5*\i+3,1+2*\j) circle (0.13cm);
    \fill (5*\i+2,4+2*\j) circle (0.2cm);
    \fill (5*\i+4,4+2*\j) circle (0.2cm);}
  \foreach \i in {2,3}
  \foreach \j in {0,1,2}
    {\fill (5*\i+1,1+2*\j) circle (0.2cm);
    \fill (5*\i+3,1+2*\j) circle (0.2cm);
    \fill (5*\i+2,4+2*\j) circle (0.3cm);
    \fill (5*\i+4,4+2*\j) circle (0.3cm);}
  \draw[shift={(1,0)}] (0pt,2pt) -- (0pt,-2pt) node[below] {$a_1$};
  \draw[shift={(3,0)}] (0pt,2pt) -- (0pt,-2pt) node[below] {$a_2$};
  \draw[shift={(6,0)}] (0pt,2pt) -- (0pt,-2pt) node[below] {$a_3$};
  \draw[shift={(8,0)}] (0pt,2pt) -- (0pt,-2pt) node[below] {$a_4$};
  \draw[shift={(11,0)}] (0pt,2pt) -- (0pt,-2pt) node[below] {$a_5$};
  \draw[shift={(13,0)}] (0pt,2pt) -- (0pt,-2pt) node[below] {$a_6$};
  \draw[shift={(16,0)}] (0pt,2pt) -- (0pt,-2pt) node[below] {$a_7$};
  \draw[shift={(18,0)}] (0pt,2pt) -- (0pt,-2pt) node[below] {$a_8$};
  \draw[shift={(2,0)}] (0pt,2pt) -- (0pt,-2pt) node[below] {$b_1$};
  \draw[shift={(4,0)}] (0pt,2pt) -- (0pt,-2pt) node[below] {$b_2$};
  \draw[shift={(7,0)}] (0pt,2pt) -- (0pt,-2pt) node[below] {$b_3$};
  \draw[shift={(9,0)}] (0pt,2pt) -- (0pt,-2pt) node[below] {$b_4$};
  \draw[shift={(12,0)}] (0pt,2pt) -- (0pt,-2pt) node[below] {$b_5$};
  \draw[shift={(14,0)}] (0pt,2pt) -- (0pt,-2pt) node[below] {$b_6$};
  \draw[shift={(17,0)}] (0pt,2pt) -- (0pt,-2pt) node[below] {$b_7$};
  \draw[shift={(19,0)}] (0pt,2pt) -- (0pt,-2pt) node[below] {$b_8$};
  \draw[shift={(0,3)}] (-3pt,0pt) -- (2pt,0pt) node[left] {$c_0$};
  \draw[shift={(0,6)}] (-3pt,0pt) -- (2pt,0pt) node[left] {$c_1$};
\end{tikzpicture}
\caption{If we choose the parameters carefully this joint distribution allows additive noise models in both directions. (Thickness stands for probability values.)} 
\label{counter_finite}
\end{minipage}
\end{center}
\end{figure}

\subsubsection{$X$ and $Y$ have infinite support}
\begin{theorem} \label{xyinfinite}
Consider an additive noise model $X \rightarrow Y$ where both $X$ and $Y$ have infinite support. We distinguish between two cases
\begin{enumerate}
\item {\bf N has compact support: }
$\exists m,l \in \Z$, such that $\supp N=[m,l]$.\\
Assume there is an ANM from $X$ to $Y$ and $f$ does not have infinitely many infinite sets, on which it is constant. Then the model is reversible $\Longleftrightarrow$ there exists a disjoint decomposition $\bigcup_{i=1}^l C_i=\supp X$ that satisfies the same conditions as in Theorem \ref{yfinite}. 
\item {\bf N has entire $\Z$ as support: } $\prob(N=k)>0 \, \forall \, k\in \Z$.\\
Suppose $X$ and $Y$ are not independent and there is an ANM $X\rightarrow Y$ and $Y\rightarrow X$. If $f$, the distribution of $N$ and all $p_X(k)$ for all $k\geq m$ for any $m \in \Z$ are known, then all other values $p_X(k)$ for $k < m$ are determined. That means even only a small fraction of the parameters determine the remaining parameters.
\end{enumerate}
\end{theorem}
Note that the first case is again a complete characterization of all instances of a joint distribution, an ANM in both directions is conform with. The second case does not yield a complete characterization, but shows how restricted we are for a given function $f$ and noise $N$ in order to choose a distribution $\prob^X$ that yields a reversible ANM.

\subsection{Cyclic Constraint} \label{sec_cyclic_id}
Again we investigate how often $X$ and $Y$ can both satisfy an additive noise model with respect to each other, this time considering cyclic constraints. Assume 
$Y=f(X)+N$ with $N \independent X$. We will show that in the generic case the model is not reversible, meaning there is no $g$ and $\tilde N$, such that $X=g(Y)+\tilde N$ with $\tilde N \independent Y$.

Note that the model $Y=f(X)+N$ is reversible if and only if there is a function $g$, such that
\begin{equation} \label{eq:id}
p(x)\cdot n\big(y-f(x)\big)= q(y) \cdot \tilde n\big(x-g(y)\big) \; \forall x,y\,,
\end{equation}
where
\begin{align*}
q(y) = \sum_{\tilde x} p(\tilde x) n\big(y-f(\tilde x)\big) \; \mbox{ and } \;
\tilde n(a)=\frac{p\big(g(\tilde y)+a\big)\cdot n\Big(\tilde y-f\big(g(\tilde y)+a\big)\Big)}{q(\tilde y)} \quad \forall \tilde y \,:\, q(\tilde y)\neq0 
\end{align*}

\subsubsection{Non-Identifiable Cases}
First, we give three examples of additive noise models that are not identifiable. This restricts the class of situations in which identifiability can be expected. Figure \ref{fig_counter_cyclic} shows instances of all three examples.

\begin{example}
$X$ and $Y$ can be independent even though there is an ANM from $X$ to $Y$:\\
(i) If $Y=f(X)+N$ and $f(k)=\mathrm{const}$ for all $k: p(k)\neq0$, then the model is reversible.\\
(ii) If $Y=f(X)+N$ for a uniformly distributed noise $N$, then the model is reversible. \vspace{-0.3cm}
\end{example}
\begin{proof}
(i) It follows $Y=N+\mathrm{const}$ with probability one. Thus $X$ and $Y$ are independent and $X=g(Y)+X$ for $g\equiv0$ is a backward model.
(ii) 
$\hat N:=f(X)+N$ is still uniformly distributed and thus independent of $X$. We thus have again $Y=\hat N+0$.
\end{proof}

\begin{example}
If $Y=f(X)+N$ for a bijective and affine $f$ and uniformly distributed $X$, then the model is reversible. \vspace{-0.3cm}
\end{example}
\begin{proof}
Since $X$ is uniformly distributed and $f(x)=ax+b$ is bijective, $Y$ is uniform, too. For $g(y)= f^{-1}(y)$ and $\tilde n(k)=n(\tilde y-f(g(\tilde y)+k))=n(b-f(k))$ equation \eqref{eq:id} is satisfied.
\end{proof}

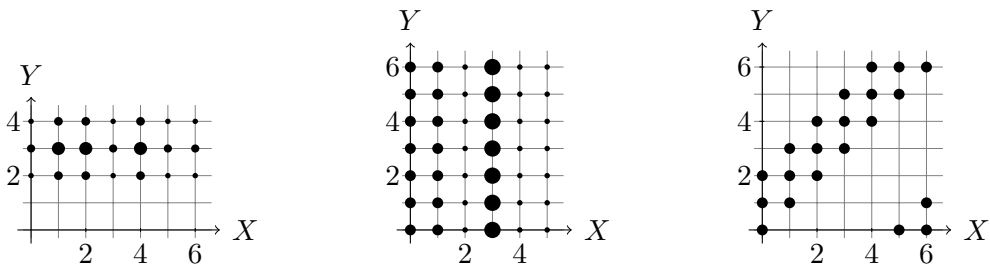
\begin{figure}[h]
\centering
\begin{tikzpicture}[scale=0.36,inner sep=1.55mm]
  \draw[gray,very thin] (-0.3,-0.3) grid (6.6,4.6);
  \draw[->] (-0.5,0) -- (6.9,0) node[right] {$X$};
  \draw[->] (0,-0.5) -- (0,4.9) node[above] {$Y$};
  \foreach \i in {0,3,5,6}
    {\fill (\i,3) circle (0.15cm);
\fill (\i,2) circle (0.1cm);
\fill (\i,4) circle (0.1cm);}
 \foreach \i in {1,2,4}
    {\fill (\i,3) circle (0.24cm);
\fill (\i,2) circle (0.16cm);
\fill (\i,4) circle (0.16cm);}
  \draw[shift={(4,0)}] (0pt,2pt) -- (0pt,-2pt) node[below] {$4$};
  \draw[shift={(2,0)}] (0pt,2pt) -- (0pt,-2pt) node[below] {$2$};
  \draw[shift={(6,0)}] (0pt,2pt) -- (0pt,-2pt) node[below] {$6$};
  \draw[shift={(0,2)}] (-3pt,0pt) -- (2pt,0pt) node[left] {$2$};
  \draw[shift={(0,4)}] (-3pt,0pt) -- (2pt,0pt) node[left] {$4$};
\end{tikzpicture}
\hspace{1.1cm}
\begin{tikzpicture}[scale=0.36,inner sep=1.55mm]
  \draw[gray,very thin] (-0.3,-0.3) grid (5.6,6.6);
  \draw[->] (-0.5,0) -- (5.9,0) node[right] {$X$};
  \draw[->] (0,-0.5) -- (0,6.9) node[above] {$Y$};
  \foreach \i in {0,1,2,3,4,5,6}
    {\fill (0,\i) circle (0.2cm);
    \fill (1,\i) circle (0.2cm);
    \fill (2,\i) circle (0.1cm);
    \fill (4,\i) circle (0.1cm);
    \fill (5,\i) circle (0.1cm);
    \fill (3,\i) circle (0.3cm);}
\draw[shift={(2,0)}] (0pt,2pt) -- (0pt,-2pt) node[below] {$2$};
\draw[shift={(4,0)}] (0pt,2pt) -- (0pt,-2pt) node[below] {$4$};
  \draw[shift={(0,2)}] (-3pt,0pt) -- (2pt,0pt) node[left] {$2$};
  \draw[shift={(0,4)}] (-3pt,0pt) -- (2pt,0pt) node[left] {$4$};
  \draw[shift={(0,6)}] (-3pt,0pt) -- (2pt,0pt) node[left] {$6$};
\end{tikzpicture}
\hspace{1.1cm}
\begin{tikzpicture}[scale=0.36,inner sep=1.55mm]
  \draw[gray,very thin] (-0.3,-0.3) grid (6.6,6.6);
  \draw[->] (-0.5,0) -- (6.9,0) node[right] {$X$};
  \draw[->] (0,-0.5) -- (0,6.9) node[above] {$Y$};
  \foreach \i in {0,1,2,3,4}
  \foreach \j in {0,1,2}
    {\fill (\i,\i+\j) circle (0.2cm);}
  \fill (6,6) circle (0.2cm);
  \fill (5,6) circle (0.2cm);
  \fill (5,5) circle (0.2cm);
  \fill (5,0) circle (0.2cm);
  \fill (6,0) circle (0.2cm);
  \fill (6,1) circle (0.2cm);

  \draw[shift={(4,0)}] (0pt,2pt) -- (0pt,-2pt) node[below] {$4$};
\draw[shift={(2,0)}] (0pt,2pt) -- (0pt,-2pt) node[below] {$2$};
  \draw[shift={(6,0)}] (0pt,2pt) -- (0pt,-2pt) node[below] {$6$};
  \draw[shift={(0,2)}] (-3pt,0pt) -- (2pt,0pt) node[left] {$2$};
  \draw[shift={(0,4)}] (-3pt,0pt) -- (2pt,0pt) node[left] {$4$};
  \draw[shift={(0,6)}] (-3pt,0pt) -- (2pt,0pt) node[left] {$6$};
\end{tikzpicture}
\caption{The joint distributions allow additive noise models in both directions
. They are instances of Examples 1(i), 1(ii) and 2 (from left to right).}
\label{fig_counter_cyclic}
\end{figure}

\subsubsection{Identifiability Results}
Motivated by the counter examples we now make the assumptions that $f$ is not constant (Example 1(i)), $N$ is not uniformly distributed (Example 1(ii)) and that not both $X$ and $Y$ are uniformly distributed (Example 2).
Without proof we state the conjecture that this is already enough to ensure identifiability meaning an ANM can only hold in one direction.   

\begin{theorem}[Conjecture]
Assume $X,Y$ and $N$ are random variables that are not uniformly distributed and non-degenerate (that is they do not only take one value). Assume further that we have an additive noise model $Y=f(X)+N, \, N \independent X$ with non-constant $f$. Then there is no additive noise model from $Y$ to $X$.
\end{theorem}

In this work we prove a slightly weaker statement. Usually the distribution $n(l)$ (similar for $p(k)$) is determined by $\tilde m-1$ free parameters. As long as the sum remains smaller than $1$, there are no (equality) constraints for the values of $n(0), \ldots, n(\tilde m-2)$. Only $n(\tilde m-1)$ is determined by $\sum_{l=0}^{\tilde m-1} n(l)=1$. We show that in the case of an reversible additive noise model the number of free parameters of the marginal $n(l)$ is heavily reduced. The exact number of constraints depends on the backward model, but can be bounded from below by 2. 
Furthermore the proof shows that a dependence between values of $p$ and $n$ is introduced.
Both of these constraints are considered to lead to non-generic models. That means for any {\it generic} choice of $p$ and $n$ we can only have an additive noise model in one direction.\\
Note further that $(\#\supp X \cdot \#\supp N)$ is the number of points $(x,y)$ that have probability greater than $0$. It must be possible to distribute these points equally to all points from $\#\supp Y$ in order to allow a backward additive noise model. Thus we have the necessary condition $\#\supp Y \,\vert\, (\#\supp X \cdot \#\supp N)$.

\begin{theorem} \label{id_cyclic}
Assume $Y=f(X)+N, \, N \independent X$ with non-uniform $X$ ($m$-cyclic), $Y$ ($\tilde m$-cyclic) and $N$ ($\tilde m$-cyclic) and non-constant $f$. 
\begin{itemize}
\item[(i)] If $\#\supp Y \,\not \vert \, (\#\supp X \cdot \#\supp N)$ then there is no additive noise model from $Y$ to $X$. 
\item[(ii)] Assume that $ \#\supp X=m, \#\supp N=\tilde m$. If there is an additive noise model from $Y$ to $X$, at least one additional equality constraint is introduced for the choice of either $p$ or $n$.
\end{itemize}
\end{theorem}

\subsection{Mixed Constraints}
With the results developed in the last two section we can cover even models with mixed constraints (for the precise conditions of ``usually'' see Section \ref{sec_cyclic_id}):\\
\begin{align*}
&Y=f(X)+N, \; N \independent X, \quad X \mbox{ cyclic},\quad Y, N \mbox{ non-cyclic}\\
\overset{2.3}{\Rightarrow} \quad &Y=f(X)+N, \; N \independent X, \quad X \mbox{ cyclic}, \quad Y, N \mbox{ cyclic}\\
\overset{3.2}{\Rightarrow} \quad  &\mbox{Usually there is no ANM } X=g(Y)+\tilde N, \; \tilde N \independent Y, \quad Y \mbox{ cyclic}, \quad X, \tilde N \mbox{ cyclic}\\
\overset{2.3}{\Rightarrow} \quad &\mbox{Usually there is no ANM } X=g(Y)+\tilde N, \; \tilde N \independent Y, \quad Y \mbox{ non-cyclic},\quad X, \tilde N \mbox{ cyclic}
\end{align*}
And, conversely:
\begin{align*}
&Y=f(X)+N, \; N \independent X, \quad X \mbox{ non-cyclic}, \quad Y, N \mbox{ cyclic}\\
\overset{2.3}{\Rightarrow} \quad &Y=f(X)+N, \; N \independent X, \quad X \mbox{ cyclic}, \quad Y, N \mbox{ cyclic}\\
\overset{3.2}{\Rightarrow} \quad  &\mbox{Usually there is no ANM } X=g(Y)+\tilde N, \; \tilde N \independent Y, \quad Y \mbox{ cyclic}, \quad X, \tilde N \mbox{ cyclic}\\
\overset{2.3}{\Rightarrow} \quad &\mbox{Usually there is no ANM } X=g(Y)+\tilde N, \; \tilde N \independent Y, \quad Y \mbox{ cyclic}, \quad X, \tilde N \mbox{ non-cyclic}
\end{align*}

\section{Practical Method for Causal Inference}
Based on our theoretical findings in Section \ref{sec_theorie} we propose the following method for causal inference (see \citet{Hoyer} for the continuous case):
\begin{enumerate}
\item Given: iid data of the joint distribution $\prob^{(X,Y)}$.
\item Regression of the model $Y=f(X)+N$ leads to residuals $\hat N$,\\
regression of the model $Y=f(X)+\tilde N$ leads to residuals $\hat{\tilde N}$.
\item If $\hat N \independent X$ and $\hat{\tilde N} \notindependent Y$ infer {\it ``$X$ is causing $Y$''},\\
if $\hat N \notindependent X$ and $\hat{\tilde N} \independent Y$ infer {\it ``$Y$ is causing $X$''},\\
if $\hat N \notindependent X$ and $\hat{\tilde N} \notindependent Y$ infer {\it ``I do not know (bad model fit)''} and\\
if $\hat N \independent X$ and $\hat{\tilde N} \independent Y$ infer {\it ``I do not know (both directions possible)''}.
\end{enumerate}
(The identifiability results show that in the generic case we will not find that both $N \independent X$ and $\tilde N \independent Y$.) This procedure requires discrete methods for regression and independence testing and we now discuss our choices.   
\subsection{Regression Method} \label{sec_regr}
Given a finite number of iid samples of the joint distribution $\prob^{(X,Y)}$ we denote the sample distribution by $\hat \prob^{(X,Y)}$. In continuous regression we usually minimize a sum consisting of a loss function (like an $\ell_2$-error) and a regularization term that prevents us from overfitting. 

{\it Regularization} of the regression function is at least in principle not necessary in the discrete case. Since we may observe many different values of $Y$ for one specific $X$ value there is no risk in overfitting. This introduces further difficulties compared to continuous regression since in principle we now have to try all possible functions from $X$ to $Y$ and compare the corresponding values of the loss function. 

Minimizing a {\it loss function} like an $\ell_p$ error is not appropriate for our purpose, either: after regression we evaluate the proposed function by checking the independence of the residuals. Thus we should choose the function that makes the residuals as independent as possible. 
Therefore we consider a dependence measure (DM) between residuals and regressor as loss function, which we denote by $\DM(\hat N, X)$.

Two problems remain:\\
(1) Assume the different $X$ values $x_1 < \ldots < x_n$ occur in the sample distribution $\hat \prob^{(X,Y)}$. Then one only has to evaluate the regression function on these values. More problematic is the range of the function. 
Since we can only deal with finite numbers, we have to restrict the range to a finite set. No matter how large we choose this set, it is always possible that the resulting function class does not contain the true function. 
But since we used the freedom of choosing an additive constant to require $n(0)>n(k)$ and $\tilde n(0)>\tilde n(k)$ for all $k\neq 0$, we will always find a sample $(X_i,Y_i)$ with $Y=f(X_i)$ if the sample size is large enough. Thus it would be reasonable to consider all $Y$ values that occur together with $X=x$ as a potential value for $f(x)$. But to even further reduce the impact of this problem we regard {\it all} values between $\min Y$ and $\max Y$ as possible values for $f$. And if there are too few samples with $X=x_j$ and the true value $f(x_j)$ is not included in $\{\min Y, \min Y+1, \ldots, \max Y\}$ we may not find the true function $f$, but the few ``wrong'' residuals do not have an impact on the independence.
In practice the following deliberation is much more relevant: 

\noindent
(2) Even if all values of the true function $f$ are one of the $m:=\#\{\min Y, \min Y+1, \ldots, \max Y\}$ considered values, the problem of checking all possible functions is not tractable: If $n=20$ and $m=16$ there are $16^{20}=2^{80}$ possible functions (the amount of particles in the universe is estimated to be $\approx 10^{79}$). We thus propose the following heuristic but very efficient procedure (the experimental results will show that it works very reliably in practice):

Start with an initial function $f^0$ that maps every value $x$ to the $y$ which occurred (together with this $x$) most often under all $y$. Iteratively we then update each function value separately. Keeping all other function values $f(\tilde x)$ with $\tilde x \neq x$ fixed we choose $f(x)$ to be the value that results in the ``most independent'' residuals. This is done for all $x$ and repeated up to $J$ times as shown in Algorithm 1. Recall that we required $n(0)\geq n(k)$ for all $k$.

\begin{algorithm}[h]
   \caption{Discrete Regression with Dependence Minimization}
   \label{alg}
\begin{algorithmic}[1]
   \STATE {\bfseries Input:} $\hat \prob(X, Y)$
   \STATE {\bfseries Output:} $f$ \vspace{0.3cm}
   \STATE $f^{(0)}(x_i):=\argmax_{y} \hat \prob(X=x_i, Y=y)$ \vspace{0.2cm}
   \REPEAT
   \STATE $j=j+1$
   \FOR{$i$ in a random ordering} 
   \STATE $f^{(j)}(x_i):=\argmin_{y} \DM \big(X,Y-f^{(j-1)}_{x_i\mapsto y}(X)\big)$
   \ENDFOR
   \UNTIL{residuals $\hat N:=Y-f^j(X)$ are independent of $X$ \hspace{0.2cm} {\bfseries or } \hspace{0.2cm} $j=J$. }
\end{algorithmic}
\end{algorithm}

In the algorithm, $f^{(j-1)}_{x_i\mapsto y}(X)$ means that we use the current version of $f^{(j-1)}$ but change the function value $f(x_i)$ to be $y$. If the $\argmax$ in the initialization step is not unique we take the largest possible $y$. We can even accelerate the iteration step if we do not consider all possible values $\{\min Y, \ldots, \max Y\}$, but only the five that give the highest values of $\hat \prob(X=x_i, Y=y)$ instead.   


\subsection{Independence Test and Dependence Measure}
Assume we are given joint iid samples $(W_i, Z_i)$ of the discrete variables $W$ and $Z$ and we want to test whether $W$ and $Z$ are independent.
In our implementation we only use Person's $\chi^2$ test (e.g. \cite{lehmann}), which is most commonly used. It computes the difference between observed frequencies and expected frequencies in the contingency table. The test statistic is known to converge towards a $\chi^2$ distribution, which is taken as an approximation even in the finite sample case.
For very few samples, however, this approximation and therefore the test will usually fail. It has been suggested (e.g. \citet{cran}) that instead of a $\chi^2$ test, Fisher's exact test \citep{lehmann} could be used if not more than $80\%$ of the expected counts are larger than 5 (``Cochran's condition'').

For a dependence measure we simply use the $p$-value of the independence test or the test statistic if the $p$-value is too small (in a computer system the $p$-value is sometimes regarded to be zero).  

\section{Experiments}
\subsection{Simulated Data.}
We first investigate the performance of our method on synthetic data sets. Therefore we simulate data from additive noise models and check whether the method is able to rediscover the true model. We showed in Section \ref{sec_theorie} that only very few examples allow a reversible ANM. 

Data sets 1a and 1b support these theoretical results. We simulate a large amount of data (1000--2000 data points) from many randomly chosen models. All models that allow an ANM in both directions are instances of our examples from above (without exception). 

Data sets 2a and 2b show how well our method performs for small data size and models that are close to non-identifiable. 

Data set 3a investigates empirically the run-time performance of our regression method and compares it with a brute-force search. 

Data set 3b shows that the method does not favor one direction if the supports of $X$ and $Y$ are of different size.
\subsubsection{Integer Constraints}
{\bf Data set 1a (identifiability).}\\
With equal probability we sample from a model with (1) $\supp X \subset \{1, \ldots, 4\}$, (2) $\supp X \subset \{1, \ldots, 6\}$, (3) $X$ binomial with parameters $(n,p)$, (4) $X$ geometric with parameter $p$, (5) $X$ hypergeometric with parameters $(M,K,N)$, (6) $X$ negative binomial with parameters $(n,p)$ or (7) $X$ Poisson with parameter $\lambda$. The parameters of these distributions and the function (with values between $-7$ and $7$) are also randomly chosen. We then consider $1000$ different models.
For each model we sample 1000 data points and apply our algorithm with $\alpha =0.05$.

The results given in Table \ref{taba:2a} show that the methods works well on almost all simulated data sets. The algorithm outputs ``bad fit in both directions'' in roughly $5\%$ of all cases, which corresponds to the chosen test level. The model is non-identifiable only in very few cases, which are shown in Table \ref{taba:2a}. All of these cases are instances of the counter examples from above. This experiment further supports our proposition that the model is identifiable in the generic case.

\begin{table}[h]
\begin{center}
{\small
\begin{tabular}[b]{r||c|c|c|c}
    	$\#$ samples &correct dir.&  wrong dir. & ``both dir. poss.'' &  ``bad fit in both dir.'' \\ \hline \hline
        total & 89.9\% & 0\% & 5.3\% & 4.8\%\\ \hline
	non-overlapping noise & -&- & 3.0\% &- \\ \hline
	$f$ constant & -& -& 2.3\% &-        
\end{tabular}}
\caption{Data Set 1a. The algorithm identifies the true causal direction in almost all cases. All models that were classified as reversible are either instances, where the noise does not ``overlap'' (i.e. $f(X)+\supp N$ are disjoint) or where $f$ is constant. For the remaining models the algorithm mistakes the residuals as being dependent in $4.8\%$ of the cases, which corresponds to the test level of $\alpha=5\%$.}
\label{taba:2a}
\end{center}
\end{table}

{\bf Data set 2a (close to non-identifiable).}\\ 
For this data set we sample from the model $Y=f(X)+N$ with 
$$
n(-2)=0.2, \, n(0)=0.5, \, n(2)=0.3 \quad \mbox{ and } \quad f(-3)=f(1)=1,\, f(-1)=f(3)=2.
$$
Depending on the parameter $r$ we sample $X$ from 
$$
p_X(-3)=0.1+\frac{r}{2}, \; p_X(-1)=0.3-\frac{r}{2} , \;p_X(1)=0.15-\frac{r}{2}, \;p_X(3)=0.45+\frac{r}{2}.
$$
For each value of the parameter $r$ ranging between $-0.2\leq r \leq 0.2$ we use 100 different samples, each of which has the size 400.

In Theorem \ref{yfinite} we proved that the ANM is reversible if and only if $r=0$. Figure
\ref{fig:2a} shows that the algorithm identifies the correct direction for $r\neq 0$. Again, the test level of $\alpha=5\%$ introduces indecisiveness of roughly the same size, which can be seen for $|r|\geq 0.15$. The number of such cases can be reduced by decreasing $\alpha$ (but would lead to some wrongly accepted backward models, too).
\begin{figure}
\begin{center}
\includegraphics[width=0.5\linewidth]{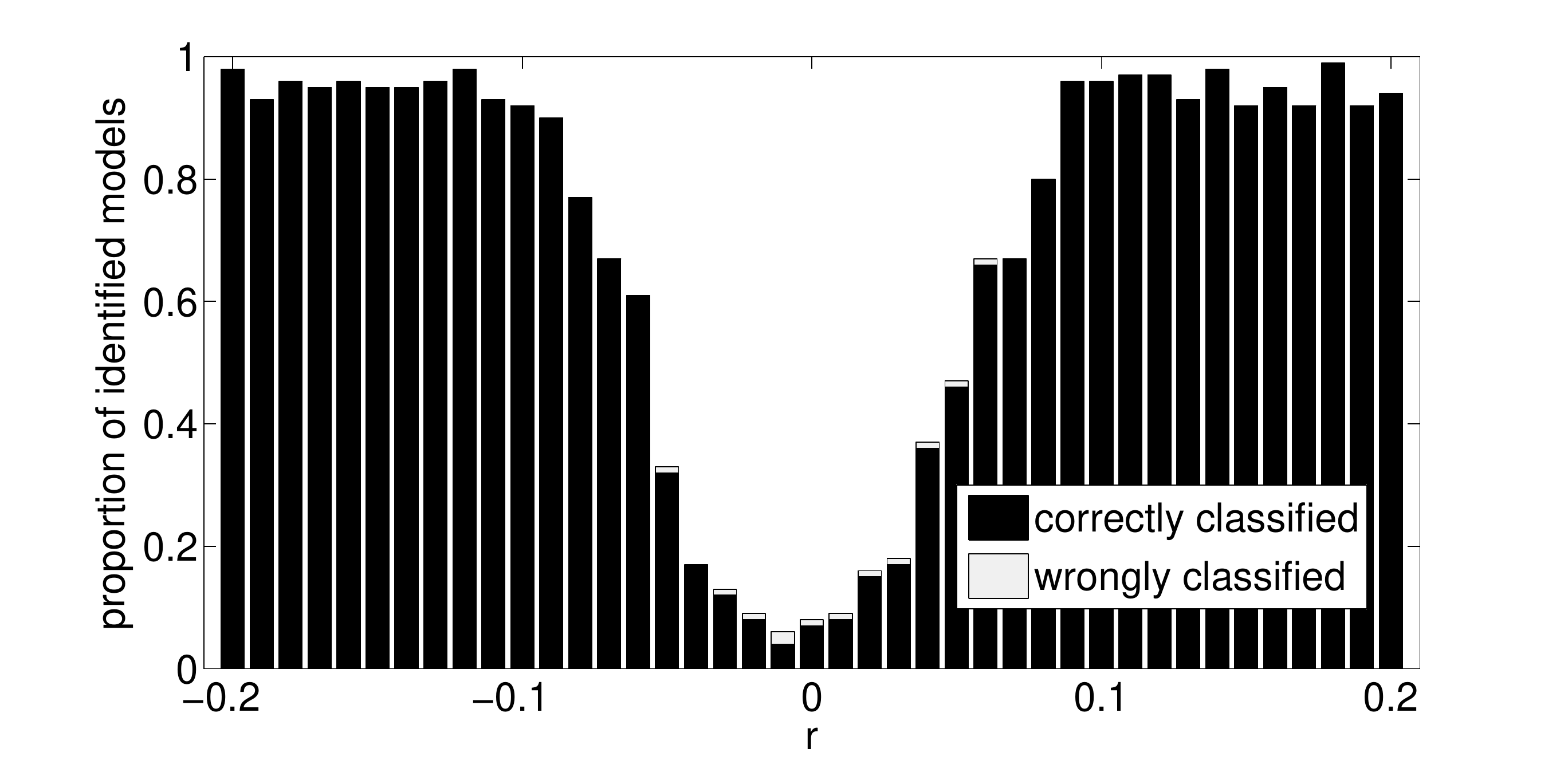}
\end{center}
\caption{Figure 1: Data set 2a. Proportion of correct and false results of the algorithm depending on the distribution of $N$. The model is not identifiable for $r=0$. If $r$ differs significantly from $0$ almost all decisions are correct.}
\label{fig:2a}
\end{figure}

{\bf Data set 3a (fast regression).}\\
The space of all functions from the domain of $X$ to the domain of $Y$ is growing very quickly in their sizes: If $\#\supp X=m$ and $\#\supp Y=\tilde m$ then the space $\mathcal F:=\{f:\supp X \rightarrow \supp Y\}$ has ${\tilde m}^m$ elements. If one of the variables has infinite support the set is even infinitely large (although this does not happen for any finite data set). It is clear that it is infeasible to optimize the regression criterion by trying every single function. As mentioned before one can argue that with high probability it is enough to only check the functions that correspond to an empirical mass that is greater than 0 (again assuming $n(0)>0$): It is likely that $\hat \prob(X=-2, Y=f(-2))>0$. We call these functions ``empirically supported''. But even this approach is often infeasible. In this experiment we compare the number of possible functions (with values between $\min Y$ and $\max Y$), the number of empirically supported functions and the number of functions the algorithm we proposed in Section \ref{sec_regr} check in order to find the true function (which it always did).

We simulated from the model $
Y=\mathrm{round} (0.5\cdot X^2)+N
$
for two different noise distributions:
$$
n_1(-2)=n_1(2)=0.05,\quad
n_1(-1)=n_1(0)=n_1(1)=0.3
$$
and
$$
n_2(-3)=n_2(3)=0.05,\quad
n_2(-2)=n_2(-1)=n_2(0)=n_2(1)=n_2(2)=0.18
$$
Each time we simulated a uniformly distributed $X$ with $i$ values between $-\frac{i-1}{2}$ and $\frac{i-1}{2}$ for $i=3, 5, 7, \ldots, 19$. For each noise/regressor distribution we simulated 100 data sets.

For $N_1$ and $i=9$, for example, there are $(11-(-2))^9 \approx 1.1\cdot 10^{10}$ possible functions in total and $5^9\approx 2.0 \cdot 10^6$ functions with positive empirical support. Our method only checked $104 \pm 33$ functions before termination. The full results are shown in Figure \ref{fig:3a}.

\begin{figure}
\begin{center}
\includegraphics[width=0.48\linewidth]{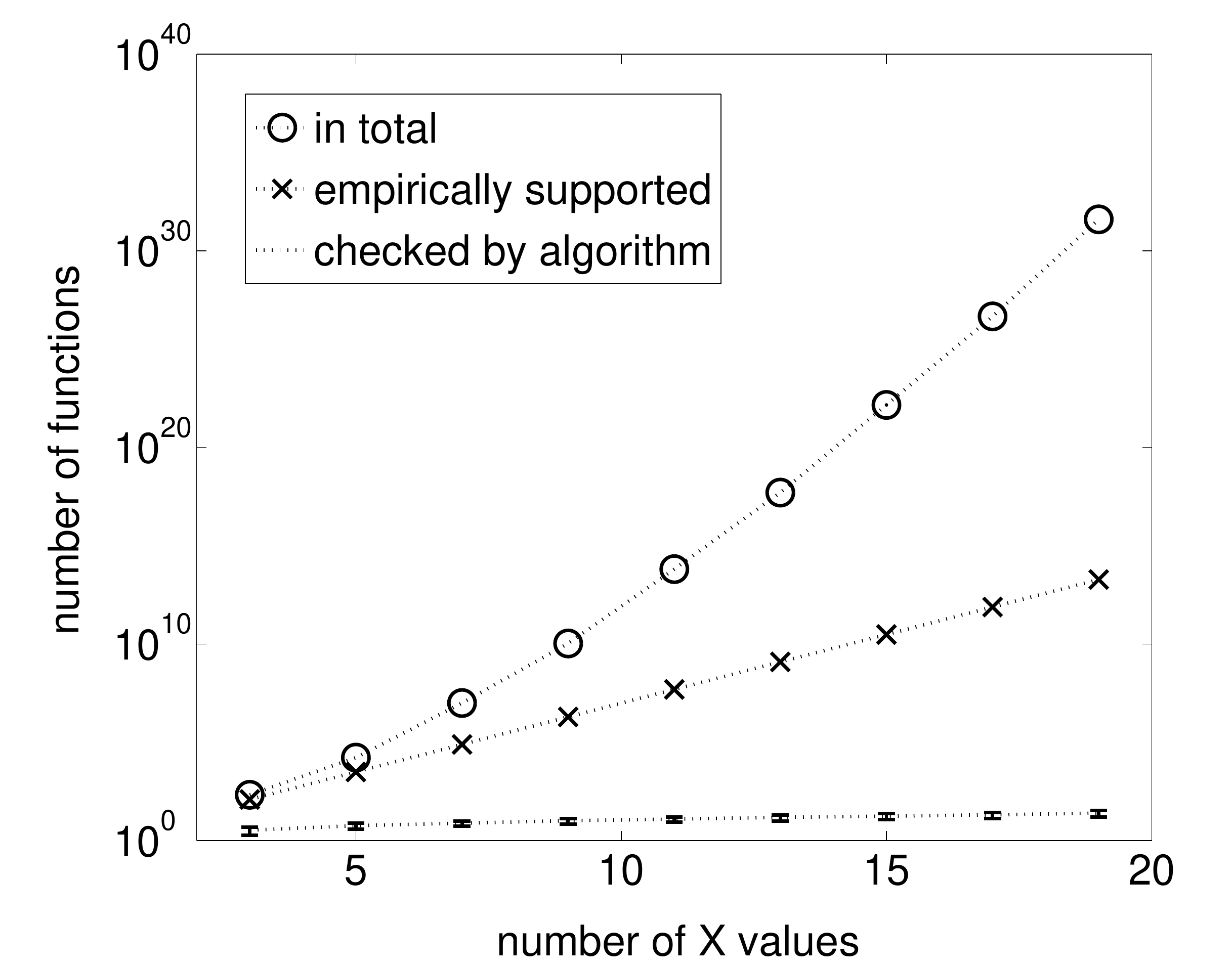}\hspace{0.02\linewidth}
\includegraphics[width=0.48\linewidth]{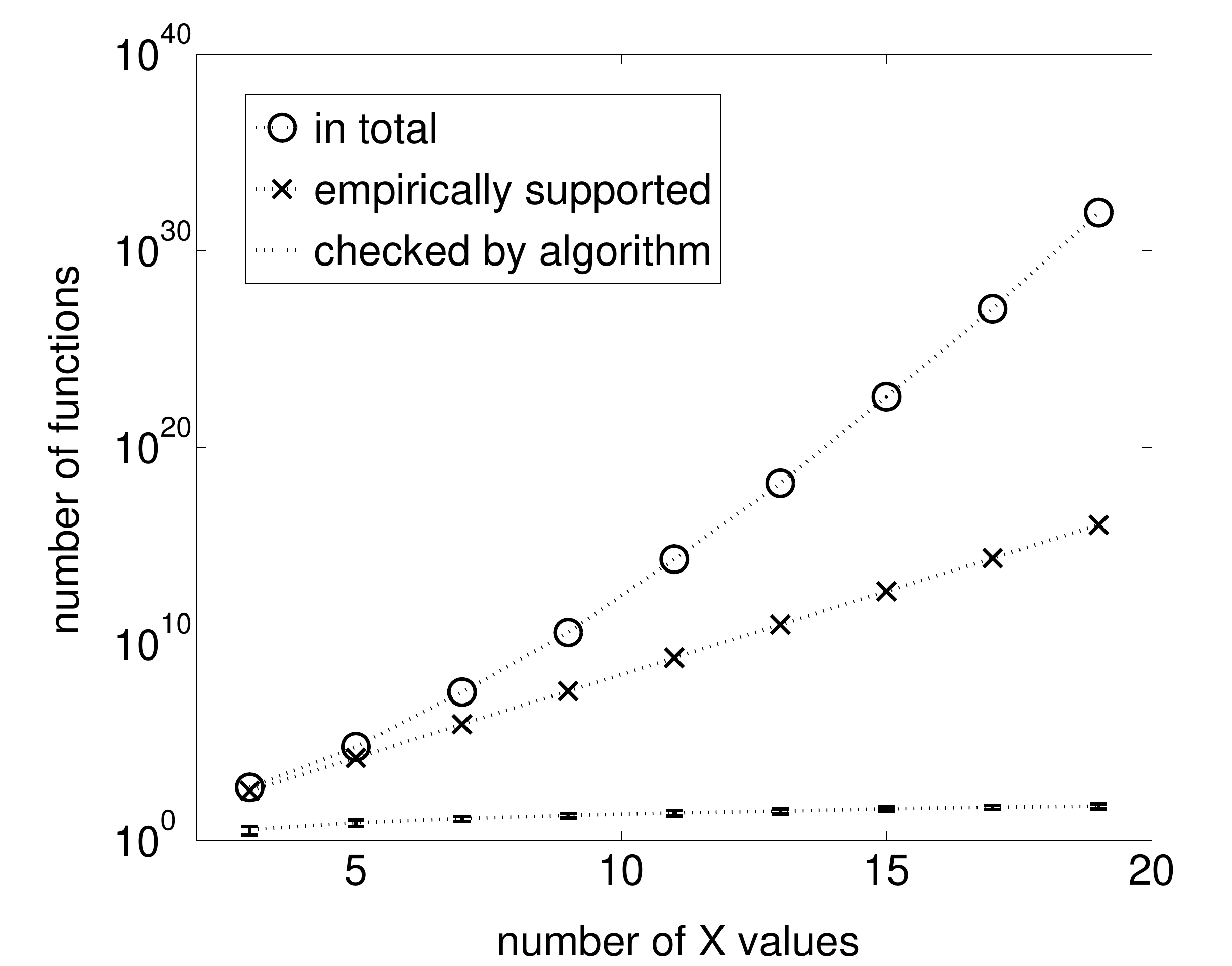}
\end{center}
\caption{Data set 3a. The size of the whole function space, the number of all functions with empirical support and the number of functions checked by our algorithm is shown for $N_1$ (left) and $N_2$ (right). An extensive search would be intractable in these cases. The algorithm we propose is very efficient and still finds always the correct function for all data sets.}
\label{fig:3a}
\end{figure}

\subsubsection{Cyclic Constraints}
{\bf Data set 1b (identifiability).}\\
For the three combinations $(m,\tilde m) \in \{(3,3),(3,5),(5,3)\}$ we consider $1000$ different models each: We randomly choose a function $f\neq \mathrm{const}$ and distributions for $p$ and $N$.
For each model we sample 2000 data points and apply our algorithm with $\alpha =0.05$.

The results given in Table \ref{tab:2} show that the method works well on almost all simulated data sets. The algorithm outputs ``bad fit in both directions'' in roughly $5\%$ of all cases, which corresponds to the chosen test level. The model is non-identifiable only in very few cases, which are shown in Table \ref{tab:2b}. All of these cases are instances of the counter examples from above. This experiment further supports our theoretical result that the model is identifiable in the generic case.

\begin{table}[h]
\begin{center}
{\small
\begin{tabular}[b]{r|l||c|c|c|c}
    	$m$ & $\tilde m$ & correct dir.&  wrong dir. & ``both dir. poss.'' &  ``bad fit in both dir.'' \\ \hline \hline
       3 & 3 & 95.3\% & 0\% & 0.8\% & 3.9\%\\ \hline
      3 & 5 & 94.8\% & 0\% & 0.0\% & 5.2\%\\ \hline
      5 & 3 & 95.5\% & 0\% & 0.3\% & 4.2\%
\end{tabular}}
\caption{Data Set 1b. The algorithm identifies the true causal direction in almost all cases. In a proportion corresponding to the test level ($\approx 5\%$), the algorithm mistakes the residuals as being dependent and thus does not find the correct model. Only in very few cases a model can be fit in both directions, which supports the results of section 3.}
\label{tab:2}
\end{center}
\end{table}

\begin{table}[h]
\begin{center}
{\small
\begin{tabular}[t]{l|c|c||c}
 function $f$& $p(1), \ldots, p(m)$ & $n(1), \ldots, n(\tilde m)$& Ex.  \\ \hline \hline
 $0\mapsto0, 1\mapsto 1, 2\mapsto 0$ & 0.27, 0.05, 0.69 & $0.30, 0.33, 0.37$&1(i)\\ \hline
 $0\mapsto0, 1\mapsto 2, 2\mapsto 0$ & 0.83, 0.00, 0.17 & $0.15, 0.26, 0.58$&1(i)\\ \hline
 $0\mapsto0, 1\mapsto 0, 2\mapsto 2$ & 0.40, 0.60, 0.00 & $0.37, 0.41, 0.22$&1(i)\\ \hline
 $0\mapsto2, 1\mapsto 0, 2\mapsto 2$ & 0.34, 0.53, 0.14 & $0.33, 0.34, 0.33$ & 1(ii)\\ \hline
 $0\mapsto0, 1\mapsto 0, 2\mapsto 2$ & 0.17, 0.74, 0.09 & $0.32, 0.33, 0.35$ & 1(ii)\\ \hline
 $0\mapsto1, 1\mapsto 0, 2\mapsto 1$ & 0.38, 0.42, 0.20 & $0.33, 0.34, 0.33$ & 1(ii)\\ \hline
 $0\mapsto1, 1\mapsto 2, 2\mapsto 2$ & 0.02, 0.86, 0.12 & $0.35, 0.30, 0.35$ &1(ii)\\ \hline
 $0\mapsto2, 1\mapsto 1, 2\mapsto 0$ & 0.33, 0.33, 0.34 & $0.85, 0.14, 0.02$&2\\ \hline \hline
 $0\mapsto1, 1\mapsto 0, 2\mapsto 1,3\mapsto 0,4\mapsto 0$ & 0.20, 0.47, 0.14, 0.08, 0.12 & $0.33, 0.33, 0.34$&1(ii)\\ \hline
$0\mapsto1, 1\mapsto 0, 2\mapsto 1,3\mapsto 1,4\mapsto 1$ & 0.55, 0.01, 0.03, 0.26, 0.14 & $0.37, 0.32, 0.31$&1(i)\\ \hline
$0\mapsto0, 1\mapsto 1, 2\mapsto 0,3\mapsto 1,4\mapsto 2$ & 0.03, 0.71, 0.06, 0.10, 0.32 & $0.32, 0.34, 0.34$&1(ii)\\ \hline


\end{tabular}}
\caption{Data Set 1b. This table shows the cases, where both directions were possible. Without exception they are instances of the examples given in section 3.}
\label{tab:2b}
\end{center}
\end{table}

{\bf Data set 2b (close to non-identifiable).}\\ 
For this data set let $m=\tilde m=4$ and $f=\mathrm{id}$. The distribution of $p$ is given by: $p(0)=0.6, p(1)=0.1, p(2)=0.1, p(3)=0.2$. Depending on the parameter $\frac{1}{2}\leq r \leq 0.8$ we sample the noise $N$ from the distribution $n(0)=n(1)=r/2, n(2)=n(3)=1/2-r/2$. That means $N$ is uniformly distributed if and only if $r=\frac{1}{2}$.

Example 1 and the fact that $n(0)\neq n(2)$ (see proof of Proposition \ref{prop:g_not_inj}) show that the model is not identifiable if and only if the noise distribution is uniform, i.e. if and only if $r=\frac{1}{2}$. The further $r$ is away from $\frac{1}{2}$, the more the noise differs from a uniform distribution and the easier it should be for our method to detect the true direction.
For each value of the parameter $r$ we use 100 different samples, each of which has size 200. Figure
\ref{fig::1} 
shows the results.

For $r=0.58$ and $r=0.68$ (indicated by the arrows in Figure \ref{fig::1}) we further investigate the dependence on the data size. Clearly, $r=0.58$ results in a model that is still very close to non-identifiability and thus we need more data to perform well (see Figure \ref{fig::2}). Note that non-identifiable models lead to very few, but not to wrong decisions.

\begin{figure}[h]
\begin{minipage}[b]{0.48\textwidth}
\begin{center}
\includegraphics[width=0.99\linewidth]{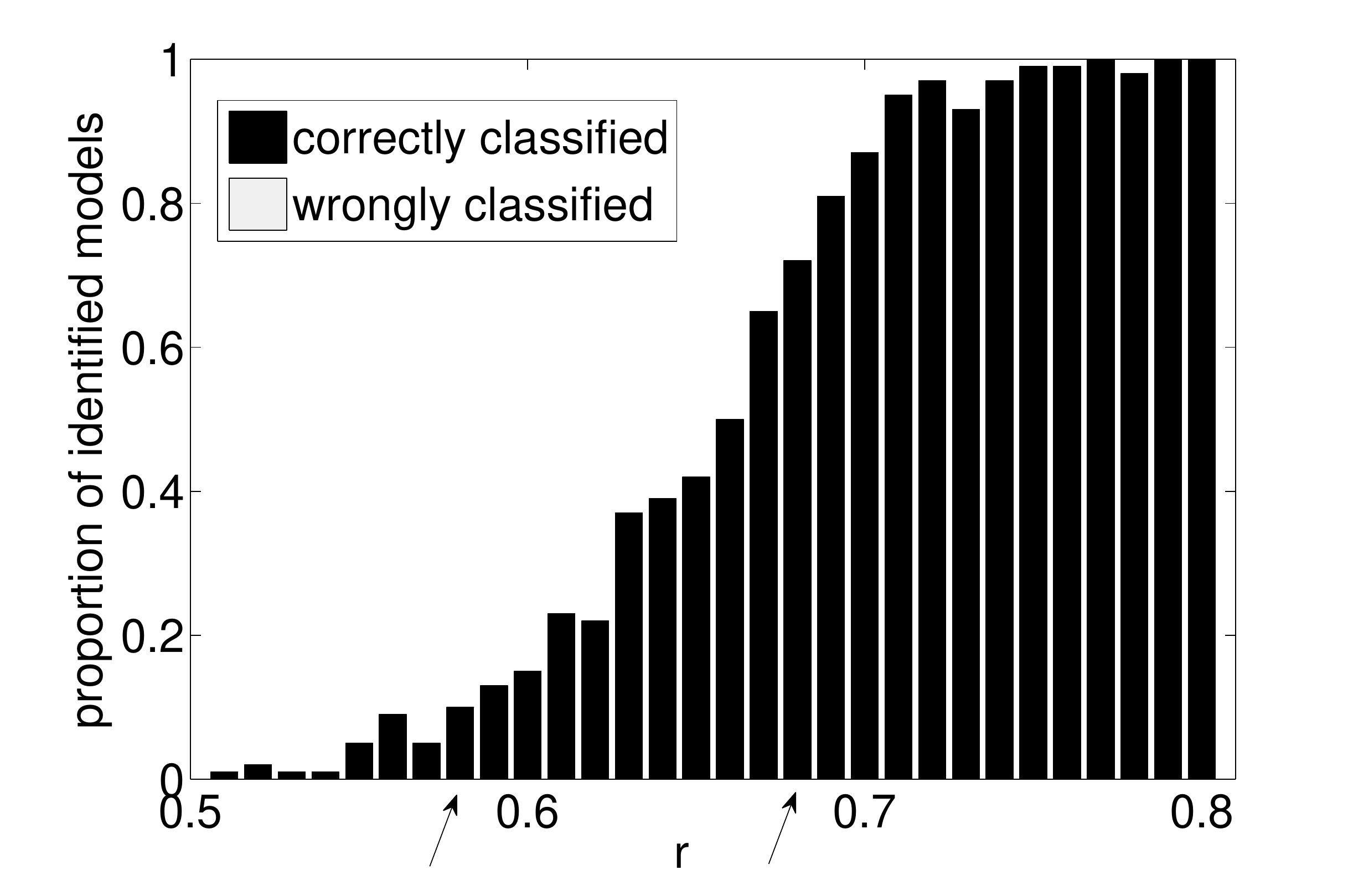}
\end{center}
\caption{Data set 2b. Proportion of correct results of the algorithm depending on the distribution of $N$ (test level $1\%$). The model is not identifiable for $r=0.5$. If $r$ differs significantly from $0.5$ the algorithm makes a decisions in almost all cases.}
\label{fig::1}
\end{minipage}
\hspace{0.5cm}
\begin{minipage}[b]{0.48\textwidth}
\begin{center}
\includegraphics[width=0.9\linewidth]{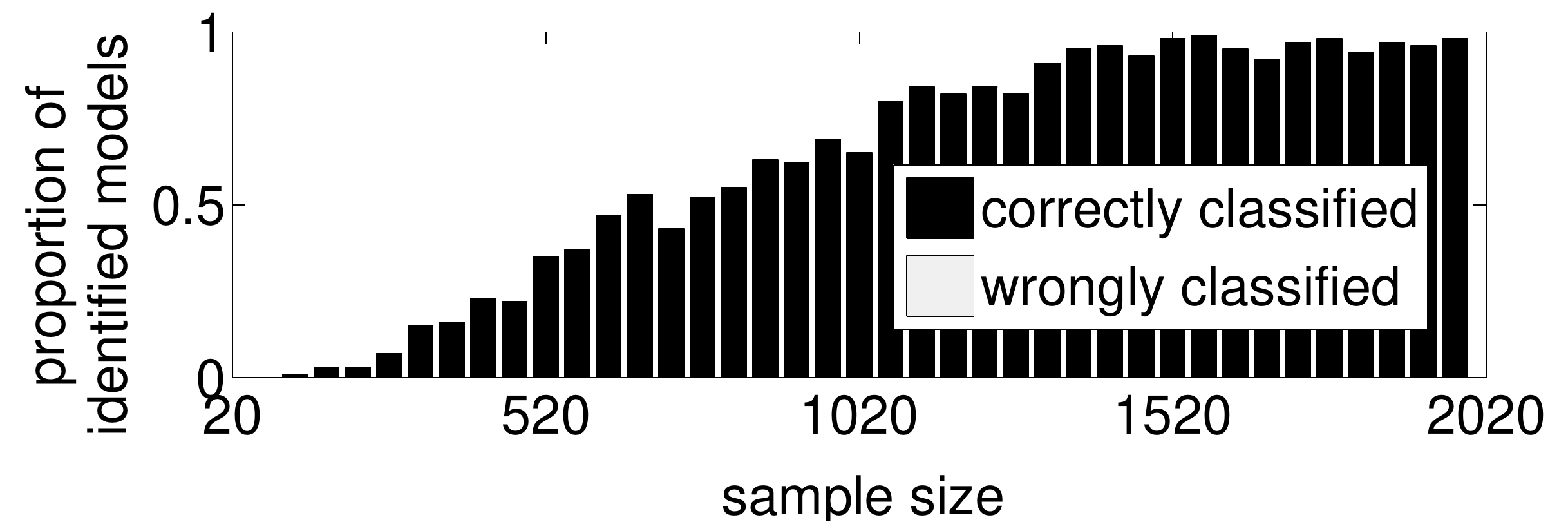}\vspace{0.2cm}\\
\includegraphics[width=0.9\linewidth]{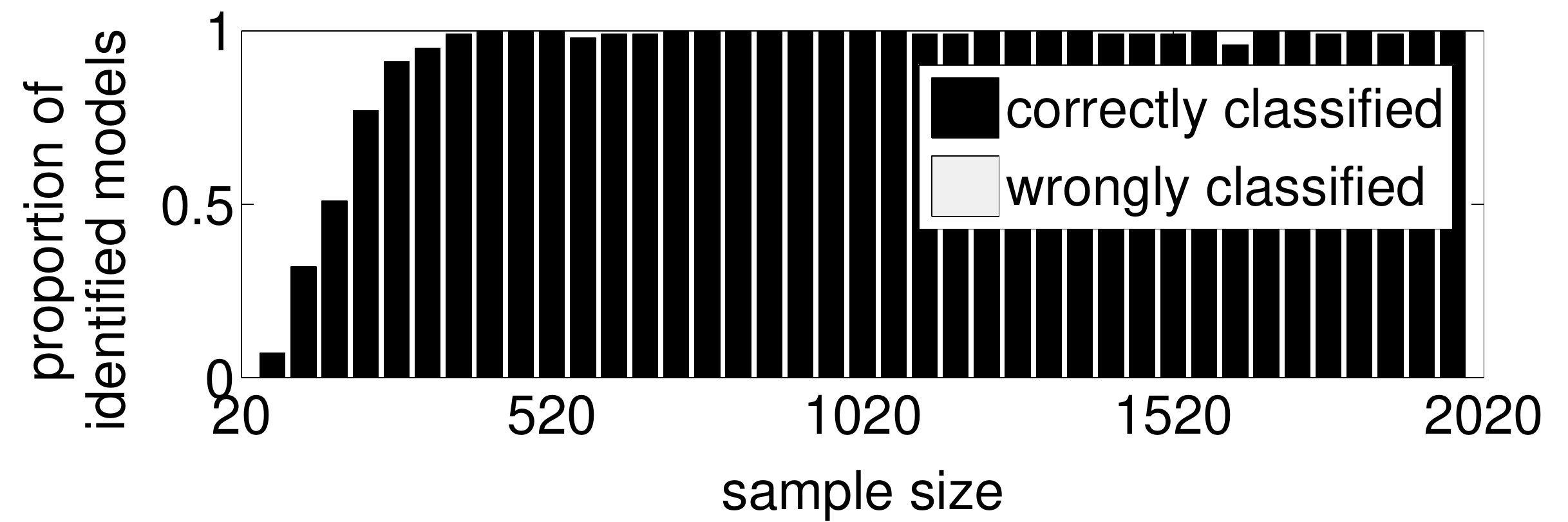}
\end{center}
\caption{For $r=0.58$ (top) and $r=0.68$ (bottom) the performance depending on the data size is shown. More data is needed if the true model is close to non-identifiable (top). In both cases the performance clearly increases with the sample size.}
\label{fig::2}
\end{minipage}
\end{figure}


{\bf Data set 3b (no direction is favored a priori).}\\
This experiment shows that our method does not favor one direction if the supports of the two random variables are very unequal in size. We choose two examples, where $m:=\#\mathcal X:=\#\supp X= 2$ and $\tilde m:=\# \mathcal Y:=\#\supp Y=10$. Since there are $2^{10}=1024$ function from $\mathcal Y$ to $\mathcal X$, but only $10^2=100$ functions from $\mathcal X$ to $\mathcal Y$ one could expect the method to favor models from $Y$ to $X$. We show that this is not the case.

For $m \neq \tilde m \in \{2,10\}$ and $m \neq \tilde m \in \{3,20\}$ we randomly choose distributions for $X$ and $N$ and a function $f$ and sampled 500 data points from this model. Table \ref{tab:33} shows that the algorithm detects the true direction in almost all cases (except if the model is non-identifiable).

\begin{table}[h]
\begin{center}
\begin{tabular}[b]{r|l||c|c|c|c}
    	$m$ & $\tilde m$ & correct dir.&  wrong dir. & ``both dir. poss.'' &  ``bad fit in both dir.'' \\ \hline \hline
       2 & 10 & 97.4\% & 0\% & 2.5\% & 0.1\%\\ \hline
      10 & 2 & 85.2\% & 0\% & 14.8\% & 0.0\%\\ \hline
      3 & 20 & 96.8\% & 0\% & 1.6\% & 1.6\%\\ \hline
      20 & 3 & 95.5\% & 0\% & 4.4\% & 0.1\%
\end{tabular}
\caption{Data Set 3b. The algorithm identifies the true causal direction in almost all cases. There is no evidence that the algorithm always favors one direction.}
\label{tab:33}
\end{center}
\end{table}

\subsection{Real Data.}
{\bf Data set 4 (abalone).}\\
We also applied our method to the {\tt abalone} data set \citep{Nash94} from the UCI Machine Learning Repository \citep{Asuncion}. We tested the sex $X$ of the abalone (male (1), female (2) or infant (0)) against length $Y_1$, diameter $Y_2$ and height $Y_3$, which are all measured in mm, and have $70, 57$ and $28$ different values, respectively. Compared to the number of samples (up to 4177) we treat this data as being discrete. Because we do not have information about the underlying continuous length we have to assume that the data structure has not been destroyed by the user-specific discretization. We regard $X\rightarrow Y_1$, $X\rightarrow Y_2$ and $X\rightarrow Y_3$ as being the ground truth, since the sex is probably causing the size of the abalone, but not vice versa. 

Clearly, the $Y$ variables do not have a cyclic structure. For the sex variable, however, the most natural model would be a structureless set which is contained in the cyclic constraints; for comparison we try both models for $X$. Our method with integer constraints is able to identify all 3 directions correctly. Since $X$ may be cyclic we also try to fit an ANM from $Y$ to $X$ with the cyclic constraints. Again, these models are rejected (see Table \ref{tab:4} and Figure \ref{fig:4}). We used $\alpha=5\%$ and the first 1000 samples of the data set.

As mentioned earlier it is not surprising that we would accept an ANM from $X$ to $Y$ even if we put cyclic constraints on $Y$ (which are certainly violated for this data set). We would obtain the following $p$-values: $p\mbox{-value}_{X\rightarrow Y1}=0.17, p\mbox{-value}_{X\rightarrow Y1}=0.19$ and $p\mbox{-value}_{X\rightarrow Y1}=0.05$. 

\begin{table}[h]
\begin{center}
{\small
\begin{tabular}[b]{c||c|c|c|c|c}
    	& $p\mbox{-value}_{X\rightarrow Y}$ & estimated function &  $p\mbox{-value}_{Y\rightarrow X}$ (non-cyclic) &  $p\mbox{-value}_{Y\rightarrow X}$ (cyclic)\\ \hline \hline
      $Y_1$& $0.17$ & $0\mapsto 39, 1\mapsto 51, 2 \mapsto 53$& $3\cdot 10^{-14}$&$3\cdot 10^{-2}$\\ \hline
      $Y_2$ &$0.19$ & $0\mapsto 30, 1\mapsto 41, 2 \mapsto 43$& $2\cdot 10^{-14}$&$4\cdot 10^{-3}$\\ \hline
      $Y_3$ &$0.05$ & $0\mapsto 10, 1\mapsto 14, 2 \mapsto 15$& $0$              &$1\cdot 10^{-8}$
\end{tabular}}
\caption{Data Set 4. The algorithm identifies the true causal direction in all 3 cases. Here, we assumed a non-cyclic structure on $Y$ and tried both cyclic and non-cyclic for $X$.}
\label{tab:4}
\end{center}
\end{table}

\begin{figure}[h]
\includegraphics[width=0.49\linewidth]{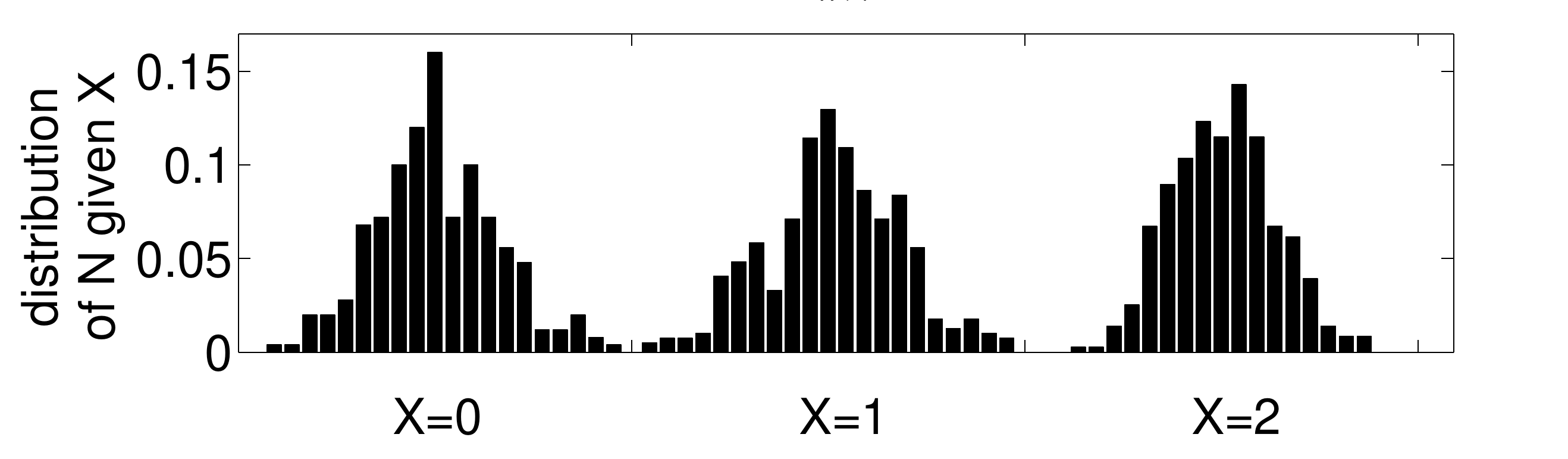}\hspace{0.02\linewidth}
\includegraphics[width=0.49\linewidth]{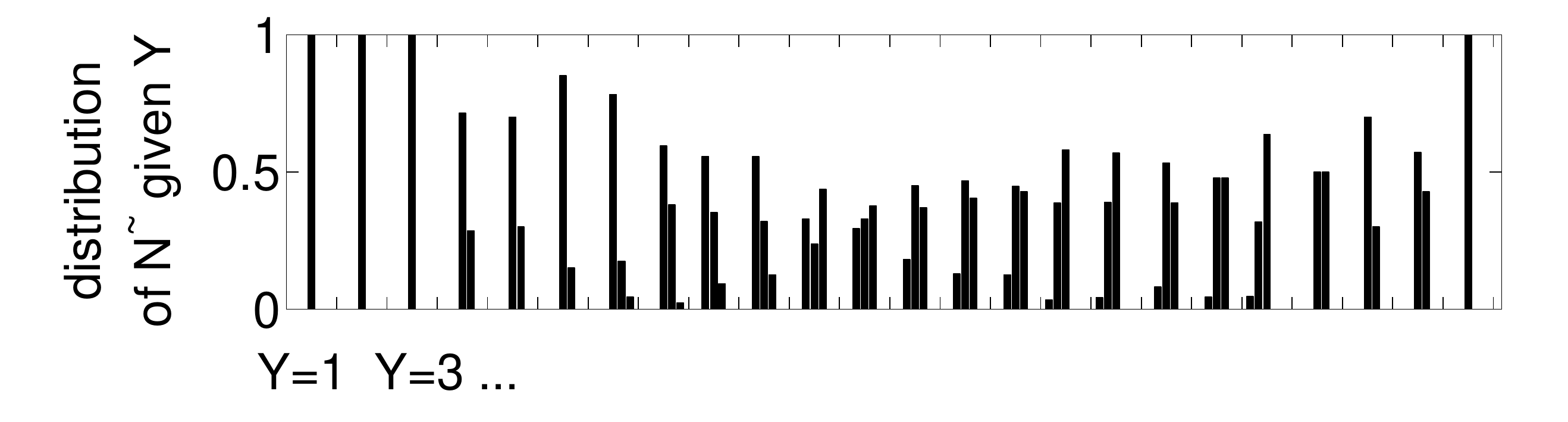}
\caption{Data Set 4. For $Y_3$ regressing on $X$ (left) and vice versa (right) the plot shows the conditional distribution of the fitted noise given the regressor. If the noise is independent, then the distribution must not depend on the regressor state. Clearly, this is only the case for $X\rightarrow Y_3$ (left), which corresponds to the ground truth.}
\label{fig:4}
\end{figure}

For this data set the method proposed by \citep{SunLauderdale} returns a slightly higher likelihood for the true causal directions than for the false directions, but this difference is so small, that the algorithm does not consider it to be significant.

The abalone data set also shows that working with $p$-values requires some carefulness. For synthetic data sets that we simulate from one fixed model the $p$-values do not depend on the data size. In real world data, however, this often is the case. If the data generating process does not exactly follow the model we assume, but is reasonable close to it, we get good fits for moderate data sizes. Only including more and more data reveals the small difference between process and model, which therefore leads to small $p$-values. Figure \ref{tab:data_size_p} shows how the $p$-values vary if we include the first $n$ data points of the abalone data set (in total: 4177). One can see that although the $p$-values for the correct direction decrease they are clearly preferable to the $p$-values of the wrong direction.
This is a well-known problem in applied statistics that also has to be considered using our method.

\begin{figure}[h]
\begin{center}
\includegraphics[width=0.32\linewidth]{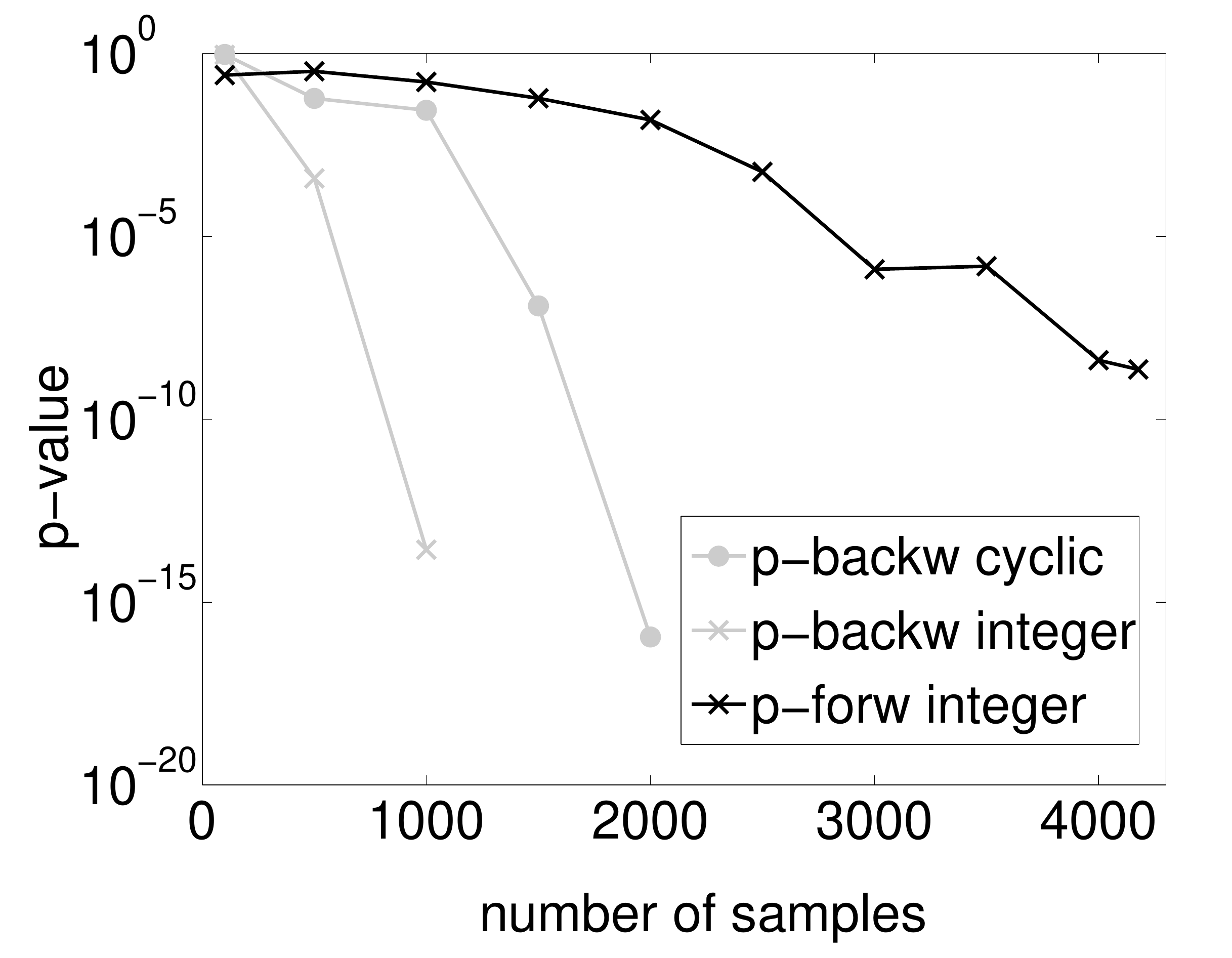}\hspace{0.01\linewidth}
\includegraphics[width=0.32\linewidth]{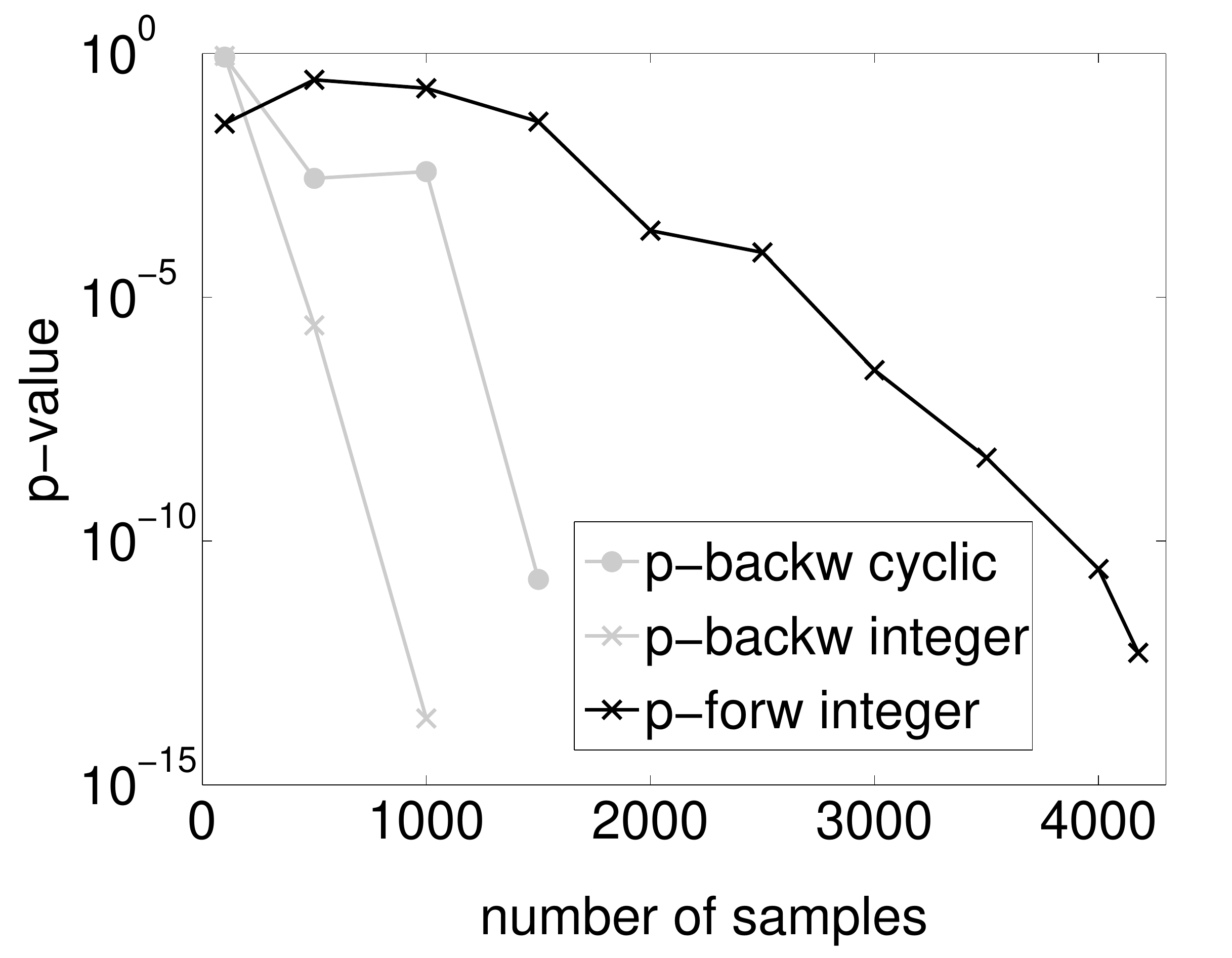}\hspace{0.01\linewidth}
\includegraphics[width=0.32\linewidth]{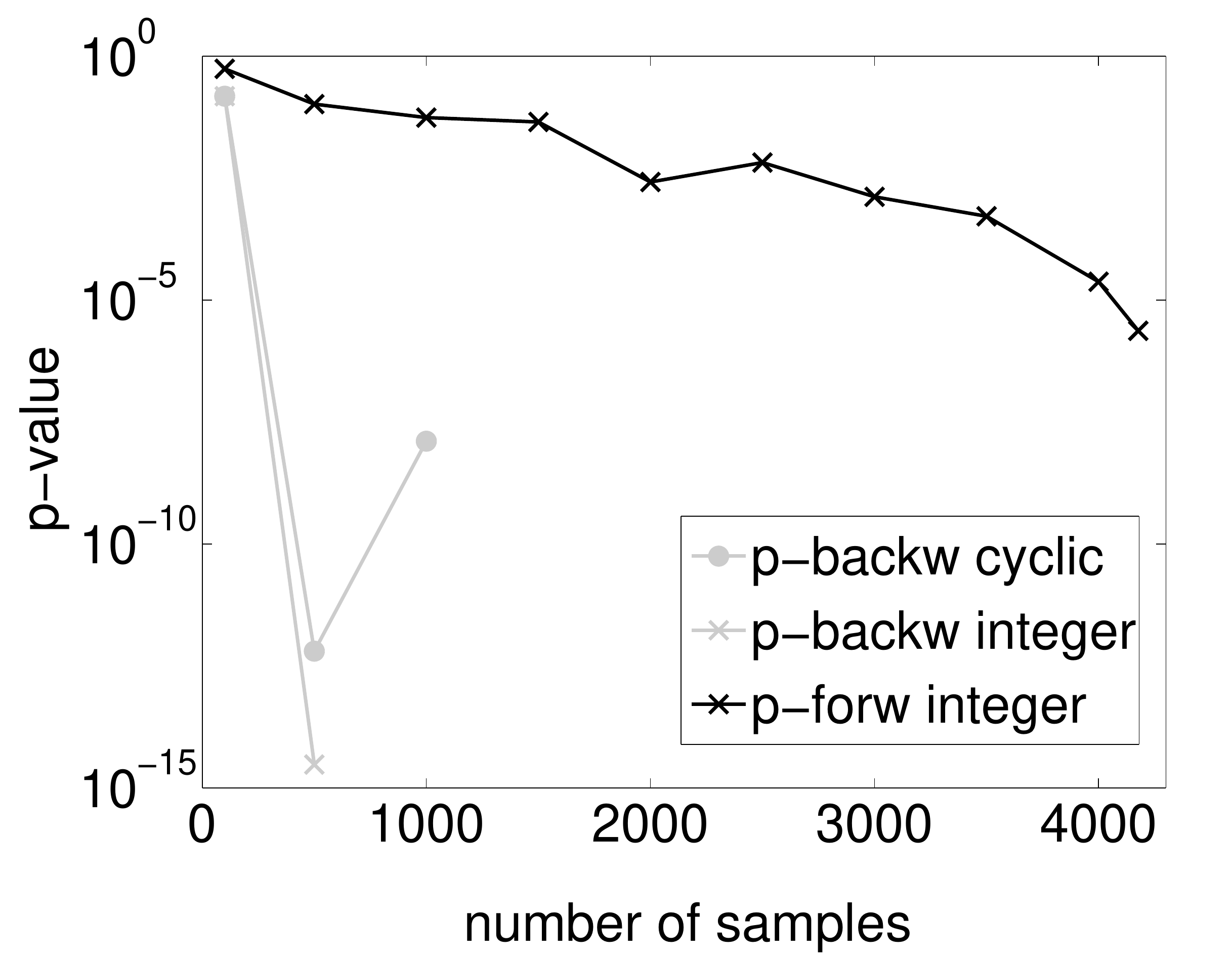}
\end{center}
\caption{Data Set 4. The plots show $p$-values of forward and backward direction depending on the number of samples we included (no data point means $p=0$). The $p$-value in the correct direction is eventually lower than any reasonable threshold. Nevertheless we prefer this direction since it is decreasing much more slowly than $p$-backward.}
\label{tab:data_size_p}
\end{figure}

{\bf Data set 5 (temperature).}\\
We further applied our method to a data set consisting of $9162$ daily values of temperature measured in Furtwangen (Germany) \citep{janzing_temp} using the variables temperature ($T$, in $^{\circ}C$) and month ($M$). Clearly $M$ inherits a cyclic structure, whereas $T$ does not. Since the month indicates the position of the earth relatively to the sun, which is surely causing the temperature on earth, we take $M \rightarrow T$ as the ground truth. Here, we aggregate states and use months instead of days. This is done in order to meet Cochran's condition and get reliable results from the independence test; it is not a scaling problem of our method (if we do not aggregate the method returns $p_{\mathrm{days}\rightarrow T}=0.9327$ and $p_{T\rightarrow \mathrm{days}}=1.0000$).

For $1000$ data points both directions are rejected ($p\mbox{-value}_{M\rightarrow T}=2e-4$, $p\mbox{-value}_{T\rightarrow M}=1e-13$). Figure \ref{fig:temp} shows, however, that again the $p\mbox{-values}_{M\rightarrow T}$ are decreasing much slower than $p\mbox{-values}_{T\rightarrow M}$ thus using other criteria than simple $p$-values we still may prefer this direction and propose it as the true one.

\begin{figure}[h]
\begin{center}
\includegraphics[width=0.4\linewidth]{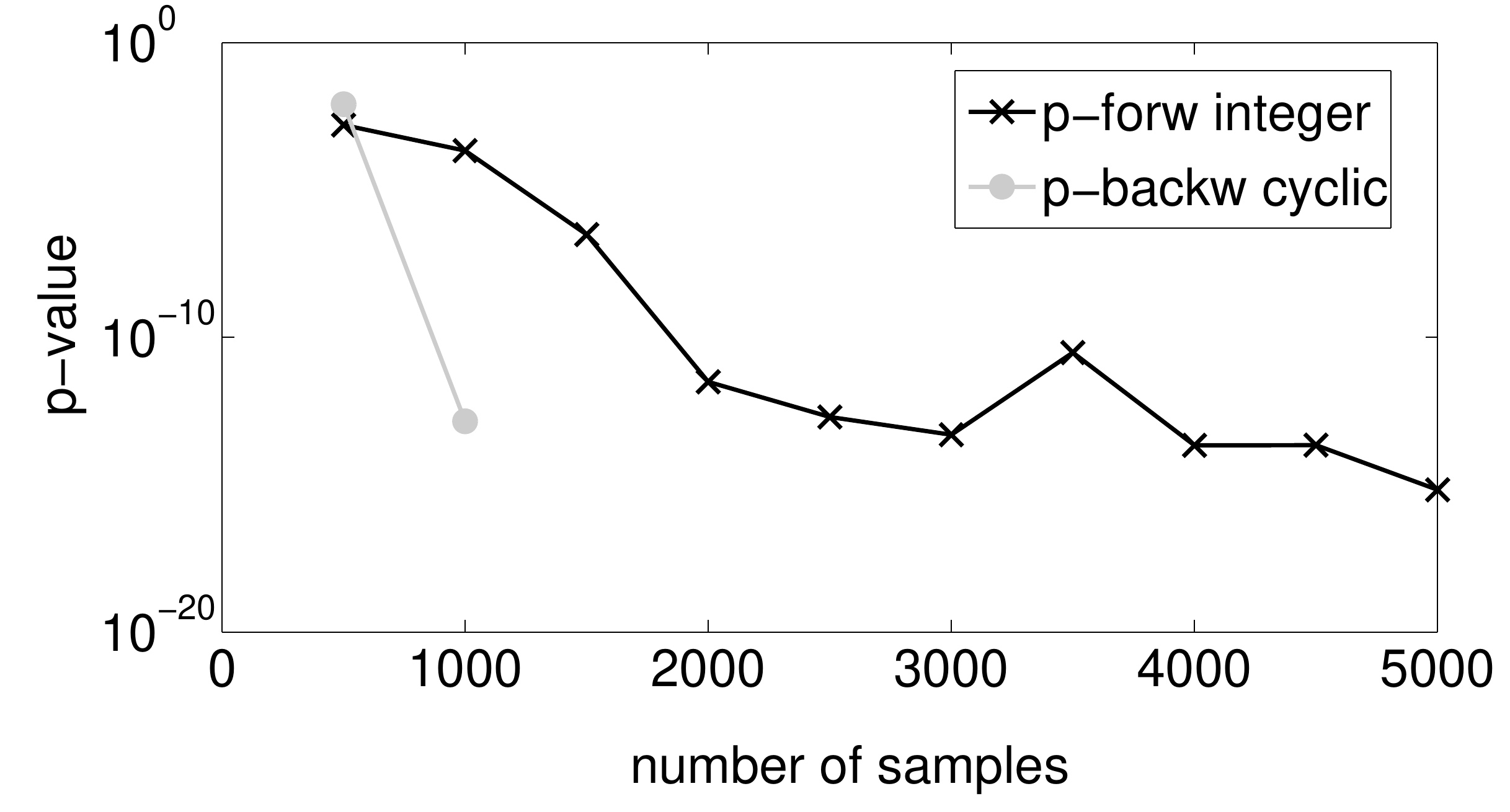}
\end{center}
\caption{Data Set 5. The plots show $p$-values of forward and backward direction depending on the number of samples we included (no data point means $p=0$). Again we prefer the correct direction since the $p$-values are decreasing much more slowly than $p$-backward.}
\label{fig:temp}
\end{figure}

\section{Conclusions and Future Work}
We proposed a method that is able to infer the cause-effect relationship between two discrete random variables. We proved that for generic choices the direction of a discrete additive noise model is identifiable in the population case and we developed an efficient algorithm that is able to infer the causal relationship between two variables for a finite amount of data. Since it is known that $\chi^2$ fails for small data sizes, changing the independence test for those cases may lead to an even higher performance of the algorithm.



Our method can be generalized in two directions: (1) handling more than two variables is straightforward from a practical point of view (although one may have to introduce regularization to make the regression computationally feasible) and (2) it should be investigated how our procedure can be applied to the case, where one variable is discrete and the other continuous. Corresponding identifiability results remain to be shown.

In future work additive noise models should be tested on more real world data sets in order to support (or disprove) additive noise models as a principle in causal inference. Furthermore we hope that more fundamental and general principles for identifying causal relationships will be developed that cover additive noise models as a special case. Nevertheless we regard our work as an important step towards understanding the difference between cause and effect.

\subsection*{Appendix}

%
\subsubsection*{A~~ Proof of Theorem \ref{yfinite}}
\begin{proof}
\begin{itemize}
\item[$\Rightarrow$:]
First we assume $\supp Y=\{y_0, \ldots, y_m\}$ with $y_0 < y_1 < \ldots < y_m$. 
Define the non-empty sets $\tilde C_i := \supp X|Y=y_i$, for $i=1, \ldots, m$. That means $\tilde C_1, \ldots, \tilde C_m \subset \supp X$ are the smallest sets satisfying $\prob(X \in \tilde C_i\,|\,Y=y_i)=1$. For all $i,j$ it follows that 
\begin{equation} \label{prop_c}
\tilde C_i=\tilde C_j \mbox{ or } \tilde C_i \cap \tilde C_j=\emptyset \mbox{ and } f\mid_{\tilde C_i}=\tilde c_i=\mathrm{const}.
\end{equation}
This is proved by an induction argument:\\
Base step: At first consider $\tilde C_m$ corresponding to the largest value $y_m$ of $\supp Y$. Assuming $f(x_1)<f(x_2)$ for $x_1, x_2 \in \tilde C_m$ leads to 
$$
y_m=f(x_1)+N_{\mathrm{max}}<f(x_2)+N_{\mathrm{max}}=y_m
$$
and therefore to a contradiction.\\
Induction step: Now consider $\tilde C_k$ and assume properties \eqref{prop_c} are satisfied for all $\tilde C_{\tilde k}$ with $k<\tilde k\leq m$. 
\begin{align*}
&\; x \in \tilde C_k \cap \tilde C_{\tilde k}\\
\Rightarrow &\; \prob(N=y_k-f(\tilde x))=\prob(N=y_k-f(x))>0 \quad \forall \tilde x \in \tilde C_{\tilde k}\\
\Rightarrow &\; \tilde C_{\tilde k} \subset \tilde C_k\\
\Rightarrow &\; \tilde C_{\tilde k} = \tilde C_k \quad \mbox{ (since $\supp \tilde N$ must have the same size for all $y$)}\\
\Rightarrow &\; f\mid_{\tilde C_k}=f\mid_{\tilde C_{\tilde k}}=\mathrm{const}
\end{align*}
Furthermore
\begin{align*}
&\; \tilde C_k \cap \tilde C_{\tilde k} = \emptyset \quad \forall k<l\leq m\\
\Rightarrow &\; f\mid_{\tilde C_k}=\mathrm{const}
\end{align*}
using the same argument as for $C_m$.\\
Thus we can choose some sets $C_0, \ldots, C_l$ from $\tilde C_0, \ldots, \tilde C_m$, where $l \leq m$, such that $C_0, \ldots, C_l$ are disjoint, and $c_k:=f(C_k)$ are pairwise different values. Wlog assume $C_0=\tilde C_0$. Further, even the sets
$$
c_k+\supp N :=\{c_k+h\,:\,\prob(N=h)>0\}
$$ 
are pairwise different: If $y_i=c_k+h_1=c_l+h_2$ then $C_k \subset \tilde C_i$ and $C_l \subset \tilde C_i$, which implies $k=l$.

Now we consider the other case, namely that $X$ has finite support. Then we define $C_0, \ldots, C_l$ to be disjoint sets, such that $f$ is constant on each of them: $c_i:=f(C_i)$. This time, it does not matter which of these sets is called $C_0$. Since 
$$
c_k+\supp N = \supp Y|X \in C_k
$$ 
we can use the same argumentation we used for $\tilde C_i$ (exchange roles of $X$ and $Y$) and deduce again that the sets $c_k+\supp N$ are disjoint. 

The rest of the proof is valid for both cases (either $X$ or $Y$ has finite support):\\
Consider $C_i$ for any $i$. According to the assumption that an additive noise model $Y\rightarrow X$ holds we have
\begin{align*}
\tilde N | Y=c_0 &\overset{d}{=} \tilde N | Y=c_i\\
\Leftrightarrow \quad X-g(c_0) | Y=c_0 &\overset{d}{=} X-g(c_i) | Y=c_i\\
\Rightarrow \quad \, \quad X +d_i | Y=c_0 &\overset{d}{=} X | Y=c_i
\end{align*}
with $d_i=g(c_i)-g(c_0)$. Thus $C_i=C_0+d_i$ (including $d_0=0$).

For $x \in C_i$ (which implies $f(x)=c_i$) 
we have
\begin{align*}
\frac{\prob(X=x)}{\prob(X \in C_i)}&=\frac{\prob(X=x)\prob(N=c_i-f(x))}{\sum_{\tilde x \in C_i}\prob(X=\tilde x)\prob(N=c_i-f(\tilde x))}\\
&=\frac{\prob(X=x,N=c_i-f(x))}{\prob(Y=c_i)}
=\prob(X=x\,|\,Y=c_i)\\
&=\prob(X=x-d_i\,|\,Y=c_0)
=\frac{\prob(X=x-d_i,N=c_0-f(x-d_i))}{\prob(Y=c_0)}\\
&=\frac{\prob(X=x-d_i)\prob(N=c_0-f(x-d_i))}{\sum_{\tilde x \in C_0}\prob(X=\tilde x)\prob(N=c_0-f(\tilde x))}
=\frac{\prob(X=x-d_i)}{\prob(X \in C_0)}
\end{align*}

\item[$\Leftarrow$:]
In order to show that we have an reversible ANM we define the function $g$ as follows
\begin{align*}
g(y) &= 0 \qquad \forall y \in c_0+\supp N\\
g(y) &= d_i \qquad \forall y \in c_i+\supp N, \, i>0
\end{align*}
The noise $\tilde N$ is determined by the joint distribution $\prob^{(X,Y)}$, of course. It remains to check, whether the distribution of
$\tilde N |Y=y$ is independent of $y$. Consider a fixed $y$ and choose $i$ such that $y \in c_i +\supp N$. Since $C_i=C_0+d_i$ the condition $g(y)+h\in C_i$ is satisfied for all $h \in C_0$ and therefore independently of $y$ and $c_i$. If $g(y)+h \notin C_i$ then $\prob(\tilde N=h\,|\,Y=y)=0$. And if $g(y)+h \in C_i$ we have
\begin{align*}
\prob(\tilde N=h\,|\,Y=y)&=\frac{\prob(X=g(y)+h, Y=y)}{\prob(Y=y)}\\
&=\frac{\prob(X=g(y)+h,N=y-f(g(y)+h))}{\prob(Y=y)}\\
&=\frac{\prob(X=g(y)+h)\prob(N=y-c_i)}{\sum_{\tilde x \in C_i}\prob(X=\tilde x)\prob(N=y-f(\tilde x))}\\
&=\frac{\prob(X=g(y)+h)}{\prob(X \in C_i)}=\frac{\prob(X=g(y)+h-d_i)}{\prob(X \in C_0)}\\
&=\frac{\prob(X=h)}{\prob(X \in C_0)}
\end{align*}
which does not depend on $y$.  
\end{itemize}
\end{proof}

\subsubsection*{B~~ Proof of Theorem \ref{xyinfinite}}
\begin{proof}
\begin{enumerate}
\item $\prob(N=k)>0 \, \forall \, m\leq k \leq l$ and $\prob(N=k)=0$, else.
\begin{itemize}
\item[$\Rightarrow$:]
\begin{figure}[h]
\begin{tikzpicture}[scale=0.38,inner sep=0.55mm]
  \draw[gray,very thin] (-0.3,-0.3) grid (29.3,9.3);
  \draw[->] (-1,0) -- (29.5,0) node[right] {$X$};
  \draw[->] (0,-1) -- (0,9.5) node[above] {$Y$};
  \foreach \i in {0,1,2,3,4,5}
  \foreach \j in {1,2,3}
    {\fill (5*\i+1,\j+\i+1) circle (0.2cm);
    \fill (5*\i+2,\j+\i) circle (0.2cm);
    \fill (5*\i+3,\j+\i) circle (0.2cm);
    \fill (5*\i+4,\j+\i+1) circle (0.2cm);}
  \draw[shift={(1,0)}] (0pt,2pt) -- (0pt,-2pt) node[below] {$x_1$};
  \draw[shift={(16,0)}] (0pt,2pt) -- (0pt,-2pt) node[below] {$x_2$};
  \draw[shift={(6,0)}] (0pt,2pt) -- (0pt,-2pt) node[below] {$\hat x_1$};
  \draw[shift={(21,0)}] (0pt,2pt) -- (0pt,-2pt) node[below] {$\hat x_2$};
  \draw[shift={(0,4)}] (-2pt,0pt) -- (2pt,0pt) node[left] {$f(x_1)+N_{\max}$};
  \draw[shift={(0,2)}] (-2pt,0pt) -- (2pt,0pt) node[left] {$f(x_1)$};
  \draw[shift={(0,5)}] (-2pt,0pt) -- (2pt,0pt) node[left] {$f(x_2)$};
  \draw[shift={(0,7)}] (-2pt,0pt) -- (2pt,0pt) node[left] {$f(x_2)+N_{\max}$};
  \draw (0.6,3.6) -- (28.4,3.6) -- (28.4,6.4) -- (0.6,6.4) -- (0.6,3.6);
  \draw (1,4) circle (0.4cm);
  \draw (4,4) circle (0.4cm);
  \draw (7,4) circle (0.4cm);
  \draw (8,4) circle (0.4cm);
  \draw (16,5) circle (0.4cm);
  \draw (19,5) circle (0.4cm);
  \draw (22,5) circle (0.4cm);
  \draw (23,5) circle (0.4cm);
  \node[circle,fill] (a) at (1,4) {};
  \node[circle,fill] (b) at (6,4) {};
  \node[circle,fill] (c) at (6,5) {};
  \node[circle,fill] (d) at (16,5) {};
  \node[circle,fill] (e) at (16,7) {};
  \path (a) edge [->,bend left=40] (b)
  (b) edge [->,bend right=40] (c)
  (c) edge [->,bend left=50] (d)
  (d) edge [->,bend right=40] (e);
\end{tikzpicture}
\label{erkl_beweis}
\end{figure}
Assume that there is an ANM in both directions $X \rightarrow Y$ and $Y\rightarrow X$. As mentioned above we have a freedom of choosing an additive constant for the regression function. In the remainder of this proof we require $\prob(N=k)=\prob(\tilde N=k)=0 \, \forall k<0$ and $\prob(\tilde N=0), \prob(N=0)>0$. The largest $k$, such that $\prob(N=k)>0$ will be called $N_{\max}$. In analogy to the proof above we define $C_y := \supp X|Y=y$ for all $y \in \supp Y$. \\
At first we note that all $C_y$ are shifted versions of each other (since there is a backward ANM) and additionally, they are finite sets. 
Otherwise this would contradict the assumptions because of the compact support of $N$.

Start with any arbitrary $x_1=\min\{f^{-1}(f(x_1))\}$ and define 
$$
\hat x_1:= \min \big\{x \in C_{f(x_1)+N_{\max}} \setminus f^{-1}(f(x_1))\big\}
$$
This implies $f(\hat x_1)>f(x_1)$ and $x_1 \in C_{f(\hat x_1)}$.
\par
\begingroup
\leftskip=1.2cm 
\noindent
If such a $\hat x_1$ does not exist because the set on the right hand side is empty, then it cannot exist for any choice of $x_1$: It is clear that $C_{f(x_1)+N_{\max}}=f^{-1}(f(x_1))$ and then we consider the first $C_{f(x_1)+N_{\max}+i}$ that is not empty. Then this set must be $f^{-1}(f(\hat x_1))$ for some $\hat x_1$. This leads to an iterative procedure and to the required decomposition of $\supp X$.
\par
\endgroup
We have that either $\max \{f^{-1}(f(\hat x_1))\} > \max \{f^{-1}(f(x_1))\}$ or $\min \{f^{-1}(f(\hat x_1))\} < \min \{f^{-1}(f(x_1))\}$: Otherwise $C_{f(\hat x_1)}$ and $C_{f(\hat x_1)-1}$ satisfy
\begin{align*}
\max C_{f(\hat x_1)-1} &\geq \max C_{f(\hat x_1)}\\
\min C_{f(\hat x_1)-1} &\leq \min C_{f(\hat x_1)}
\end{align*}
Because of $\hat x_1 \in C_{f(\hat x_1)}, \hat x_1 \notin C_{f(\hat x_1)-1}$ this contradicts the existence of an backward additive noise model. Wlog we therefore assume $\max \{f^{-1}(f(\hat x_1))\} > \max \{f^{-1}(f(x_1))\}$. Then we even have $\hat x_1>x_1$,
$$ 
x_1=\min\{C_{f(x_1)+N_{\max}}\}
$$
and
$$ 
\hat x_1=\min\{C_{f(x_1)+N_{\max}+1}\}
$$
(Otherwise we can use the same argument as above with $C_{f(x_1)+N_{\max}}$ and $C_{f(x_1)+N_{\max}+1}$.)
Define further
$$
x_2:= \min f^{-1}(f(x_1)+N_{\max}+1)
$$
Since $f^{-1}(f(x_1)) \subset C_{f(x_1)+N_{\max}}$, but $f^{-1}(f(x_1)) \cap C_{f(x_1)+N_{\max}+1}=\emptyset$, such a value must exist. Again, we can define $\hat x_2$ in the same way as above.

Set $y_1:=f(x_1)+N_{\max}$ and $\hat y_1:=f(x_1)+2\cdot N_{\max}$ and consider the finite box from $(\min C_{y_1}, y_1)$ to $(\max C_{y_2}, y_2)$. This box contains all the support from $X\,|\,Y=f(x_1)+N_{\max}+i$, where $i=1, \ldots, N_{\max}$. Assume we know the positions in this box, where $\prob^{(X,Y)}$ is greater than zero. Then this box determines the support of $X\,|\,Y= f(x_1)+2\cdot N_{\max}+1$ (the line above the box) just using the support of $N$ and $\tilde N$. Iterating gives us the whole support of $\prob^{(X,Y)}$ in the box above (from $y=f(x_2)+N_{\max}$ to $y=f(x_2)+2\cdot N_{\max}$). Since the width of the boxes are bounded by $3\cdot \max C_{f(x_1)}-\min C_{f(x_1)}$, for example, at some point the box of $x_n$ must have the same support as the one of $x_1$. Figure \ref{erkl_beweis} shows an example, in which $n=2$. Using the whole distributions of $N$ and $\tilde N$ we can now determine a factor $\alpha$ with
$$
\prob(X=x_1, Y=f(x_1)+N_{\max})=\alpha \cdot \prob(X=x_n, Y=f(x_n)+N_{\max})
$$
over the sequence $(x_1, \hat x_1, x_2, \hat x_2, \ldots, x_n)$. But since we computed the boxes in a deterministic way, the same $\alpha$ satisfies 
$$
\prob(X=x_n, Y=f(x_n)+N_{\max})=\alpha \cdot \prob(X=x_{2n-1}, Y=f(x_{2n-1})+N_{\max})
$$
and therefore 
$$
\prob(X=x_1, Y=f(x_1)+N_{\max})=\alpha^k \cdot \prob(X=x_{(k+1)n-k}, Y=f(x_{(k+1)n-k})+N_{\max})
$$
Note that a corresponding equation with the same constant $\alpha$ holds for the opposite direction. This leads to a contradiction, since there is no probability distribution for $X$ with infinite support that can fulfill this condition (no matter if $\alpha$ is greater, equal or smaller than $1$).

\item[$\Leftarrow$:]
This direction is proved in exactly the same way as in Theorem \ref{yfinite}. 
\end{itemize}

\item $\prob(N=k)>0 \, \forall \, k\in \Z$.\\
Since $X$ and $Y$ are dependent there are $y_1$ and $y_2$, such that $g(y_1)\neq g(y_2)$. Comparing $\prob(X=k, Y=y_1)$ and $\prob(X=k, Y=y_2)$ for $k\geq m$, we can identify the difference $d:=g(y_2)-g(y_1)$. If $d<0$ we can use $\frac{\prob(X=m-d-1,Y=y_1)}{\prob(X=m-d,Y=y_1)}$ and $\prob(X=m,Y=y_2)$ in order to determine $\prob(X=m-1,Y=y_2)$. If $d>0$ we use $\frac{\prob(X=m+d-1,Y=y_2)}{\prob(X=m+d,Y=y_2)}$ and $\prob(X=m,Y=y_1)$ in order to determine $\prob(X=m-1,Y=y_1)$.
In both cases this yields $\prob(X=m-1)$ and with an iterative procedure all $\prob(X=x)$.
\end{enumerate}
\end{proof}

\subsubsection*{C~~ Proof of Theorem \ref{id_cyclic}}
regarding $(i)$: Each distribution $Y\,|\,X=x_j$ has to have the same support (up to an additive shift) and thus the same number of elements with probability greater than $0$:
$$
\#\supp X\cdot \#\supp N=k\cdot\#\supp Y 
$$

\noindent For $(ii)$ we now consider 3 different cases and show necessary conditions for reversibility each.
\begin{enumerate}
\item[Case 1:] $f$ and $g$ are bijective.
\begin{prop} \label{thm:bij}
Assume $Y=f(X)+N, \, N \independent X$ for bijective $f$ and $n(l)\neq 0, p(k)\neq 0 \,\forall k,l$. If the model is reversible with a bijective $g$, then $X$ and $Y$ are uniformly distributed.
\end{prop}
  
\begin{proof}
Since $g$ is bijective we have that $\forall y \exists t_y\,:\, g(t_y)=g(y)-1$. It follows from \eqref{eq:id}
$$
\frac{n\big(y-f(x+1)\big)p(x+1)}{n\big(t_y-f(x)\big)p(x)}=\frac{\tilde n\big(x+1-g(y)\big)q(y)}{\tilde n\big(x+1-g(y)\big)q(t_y)}
$$
This implies 
\begin{equation*}
\frac{p(x+1)}{p(x)}=\frac{n\big(t_y-f(x)\big)q(y)}{n\big(y-f(x+1)\big)q(t_y)} \; \mbox{ and } \; 1=\frac{p(x+m)}{p(x)}=\frac{\prod_{k=0}^{m-1} n\big(t_y-f(x+k)\big)q(y)^m}{\prod_{k=0}^{m-1} n\big(y-f(x+k+1)\big)q(t_y)^m}
\end{equation*}
Since $f$ is bijective it follows that $q(y)=q(t_y)$. This holds for all $y$ and thus $Y$ and $X$ are uniformly distributed.
\end{proof}

\item[Case 2:] $g$ is not injective.\\
Assume $g(y_0)=g(y_1)$. 
From \eqref{eq:id} it follows that
\begin{equation}\label{eq:g_not_inj}
\frac{n\big(y_0-f(x)\big)}{n\big(y_1-f(x)\big)}=\frac{q(y_0)}{q(y_1)} \,  \forall x \; \mbox { and } \frac{n\big(y_0-f(x)\big)}{n\big(y_1-f(x)\big)}=\frac{n\big(y_0-f(\tilde x)\big)}{n\big(y_1-f(\tilde x)\big)} \; \forall x, \tilde x,
\end{equation}
which imply equality constraints on $n$. To determine the number of constraints we define a function that maps the arguments of the numerator to those of the denominator
$$
h_{y_0,y_1,f}: \begin{array}{rcl} \im(y_0-f) &\rightarrow &\Z/{\tilde m}\Z\\
y_0-f(x)&\mapsto&y_1-f(x)
\end{array}.
$$
We say $h$ has a cycle if there is a $z\in \N$, s.t. $h^k(a)=(h\circ \ldots \circ h)(a)\in \im(y_0-f) \, \forall k\leq z$ and $h^z(a)=a$. For example: $2\overset{h}{\mapsto}4\overset{h}{\mapsto}6\overset{h}{\mapsto}0\overset{h}{\mapsto}2$.
\begin{prop} \label{prop:g_not_inj}
Assume $Y=f(X)+N, \, N \independent X$ and $n(l)\neq 0, p(k)\neq 0 \,\forall k,l$. Assume further that the model is reversible with a non-injective $g$.
\begin{itemize}
\item If $h$ has only cycles, $\im f-\#cycles+1
$ parameters of $n$ are fixed.
\item Otherwise $\im f-\#cycles
$ parameters of $n$ are fixed.
\end{itemize}
\end{prop}
\begin{proof}
Assume $h$ has a cycle, then $\frac{q(y_0)}{q(y_1)}=1$ and $n\big(y_0-f(x)\big)=n\big(y_1-f(x)\big)\,\forall x$. For any non-cyclic structure of length $r$ (e.g. $3\mapsto5\mapsto7$ and $7 \notin \im(y_0-f)$), $r-1$ values of $n$ are determined, but for cycles of length $r$ (e.g. $0\mapsto2\mapsto4\mapsto6\mapsto0$) we get only $r-2$ independent equations. Together with the normalization these are $\im f-\#cycles$ (equality) constraints.
If $h$ has no cycle, we have $\im f-1$ independent equations plus the sum constraint. E.g.:
$
\frac{n(2)}{n(4)}=\frac{n(4)}{n(6)}=\frac{n(3)}{n(5)}
$ 
implies
$
n(4)=n(6)\frac{n(3)}{n(5)} \; \mbox{ and }\; n(2)=\frac{n(4)^2}{n(6)}\,.
$

Further,
$$
\frac{n\big(y_0-f(x)\big)}{n\big(y_1-f(x)\big)}=\frac{q(y_0)}{q(y_1)}=\frac{\sum_{\tilde x}p(\tilde x)n(y_0-f(\tilde x)}{\sum_{\tilde x}p(\tilde x)n(y_1-f(\tilde x))}$$
introduces a functional relationship between $p$ and $n$.
\end{proof}
Note that if $\tilde m$ does not have any divisors, there are no cycles and thus $\#\im f
$ parameters of $n$ are determined.
\begin{cor}
In all cases the number of fixed parameters is lower bounded by $\max(\lceil 1/2\cdot \im f \rceil, 2) 
 \geq 2$.
\end{cor}

\item[Case 3:] $f$ is not injective.\\
Assume $f(x_0)=f(x_1)$. In a slight abuse of notation we write 
$$
g-g: \begin{array}{ccl} \Z/{\tilde m}\Z \times \Z/{\tilde m}\Z &\rightarrow &\Z/m\Z\\
(y, \tilde y)&\mapsto&g(y)-g(\tilde y)
\end{array}.
$$
Similar as above, we define 
$$
h_{x_0,x_1,g}: \begin{array}{rcl} \im\big(x_0-(g-g)\big) &\rightarrow &\Z/m\Z\\
x_0-g(y)+g(\tilde y)&\mapsto&x_1-g(y)+g(\tilde y)
\end{array}.
$$
We say that $h$ has a cycle if there is a $z\in \N$, s.t. $h^k(a)=(h \circ \ldots \circ h)(a)\in \im\big(x_0-(g-g)\big) \forall k\leq z$ and $h^z(a)=a$. 
\begin{prop} \label{prop:f_not_inj}
Assume $Y=f(X)+N, \, N \independent X$, $f$ is not injective and $n(l)\neq 0, p(k)\neq 0 \,\forall k,l$. Assume further that the model is reversible for a function $g$.
\begin{itemize}
\item If $h$ has only cycles, $\im (g-g)-\#cycles+1
$ parameters of $p$ are fixed.
\item Otherwise $\im (g-g)-\#cycles
$ parameters of $p$ are fixed.
\end{itemize}
\end{prop}
\begin{proof}
From \eqref{eq:id} it follows that
\begin{align*}\label{eq:g_not_inj}
\frac{p(x_0)}{p(x_1)}&=\frac{\tilde n\big(x_0-g(y)\big)}{\tilde n\big(x_1-g(y)\big)}
=\frac{p\big(x_0-g(y)+g(\tilde y)\big)\cdot n\Big(\tilde y -f\big(x_0-g(y)+g(\tilde y)\big)\Big)}{p\big(x_1-g(y)+g(\tilde y)\big)\cdot n\Big(\tilde y -f\big(x_1-g(y)+g(\tilde y)\big)\Big)} \;\forall y,\tilde y
\end{align*}
The rest follows analogously to the proof of Proposition \ref{prop:g_not_inj}.
\end{proof}

If $(x_1-x_0)$ does not divide $m$, there are no cycles and thus $\im (g-g)
$ parameters of $p$ are determined.
\begin{cor}
In all cases the number of fixed parameters is lower bounded by $\max(\lceil 1/2\cdot \im (g-g)\rceil, 2)
 \geq 2$.
\end{cor}
\end{enumerate} 
Note that these three cases are sufficient since $f$ and $g$ injective implies $n=m$ and $f$ and $g$ bijective.

\bibliographystyle{natmlapa}

\bibliography{bibliography}

\end{document}